\pdfoutput=1 

\PassOptionsToPackage{linktocpage=true}{hyperref}

\documentclass[notheorems]{colt2023} 

\usepackage{blindtext}
\usepackage[utf8]{inputenc} 
\usepackage{cmap}
\usepackage[T1]{fontenc}    

\hypersetup{colorlinks=true, urlcolor=blue}

\usepackage{multicol}
\usepackage{multirow}
\usepackage{silence}
\usepackage{amsmath,amsfonts,amssymb,thmtools}
\usepackage{mathtools}
\usepackage{xparse}
\usepackage{xargs}
\usepackage{enumitem} 
\usepackage{etoolbox}
\usepackage{hhline}
\usepackage{mathrsfs}	
\usepackage{dsfont}
\usepackage{tikz}			
\usepackage{mathtools}
\usepackage{complexity}
\usepackage{svg}
\usepackage{nag} 					
\usepackage{makecell}
\usepackage{arydshln} 
\usepackage{booktabs} 
\usepackage{soul}
\usepackage{times}

\usepackage[capitalise,nameinlink,noabbrev]{cleveref} 

\makeatletter
  \newcommand{\jmlrBlackBox}{\rule{1.5ex}{1.5ex}}
  
  \newcommand{\jmlrQED}{\hfill\jmlrBlackBox\par\bigskip}
\providecommand{\proofname}{Proof}
  \newenvironment{proof}%
  {%
   \par\noindent{\bfseries\upshape \proofname\ }%
  }%
  {\jmlrQED}
  \newcommand*{\theorembodyfont}[1]{%
    \renewcommand*{\@theorembodyfont}{#1}%
  }
  \newcommand*{\@theorembodyfont}{\normalfont\itshape}%
  \newcommand*{\theoremheaderfont}[1]{%
    \renewcommand*{\@theoremheaderfont}{#1}%
  }
  \newcommand*{\@theoremheaderfont}{\normalfont\bfseries }%
  \newcommand*{\theoremsep}[1]{%
    \renewcommand*{\@theoremsep}{#1}%
  }
  \newcommand*{\@theoremsep}{}%
  \newcommand*{\theorempostheader}[1]{%
    \renewcommand*{\@theorempostheader}{#1}%
  }
  \newcommand*{\@theorempostheader}{}%
  \let\jmlr@org@newtheorem\newtheorem
  \renewcommand*{\newtheorem}{\@ifstar\jmlr@snewtheorem\jmlr@newtheorem}
  \newcommand*{\jmlr@snewtheorem}[2]{%
    \cslet{jmlr@thm@#1@body@font}{\@theorembodyfont}%
    \cslet{jmlr@thm@#1@header@font}{\@theoremheaderfont}%
    \cslet{jmlr@thm@#1@sep}{\@theoremsep}%
    \cslet{jmlr@thm@#1@postheader}{\@theorempostheader}%
    \newenvironment{#1}%
    {%
      \trivlist
        \item
        [%
          \hskip\labelsep{\csuse{jmlr@thm@#1@header@font}#2%
            \csuse{jmlr@thm@#1@postheader}%
          }%
        ]%
        \mbox{}\csuse{jmlr@thm@#1@sep}%
        \csuse{jmlr@thm@#1@body@font}%
    }%
    {%
      \endtrivlist
    }%
  }
  \newcommand{\jmlr@newtheorem}[1]{%
    \cslet{jmlr@thm@#1@body@font}{\@theorembodyfont}%
    \cslet{jmlr@thm@#1@header@font}{\@theoremheaderfont}%
    \cslet{jmlr@thm@#1@sep}{\@theoremsep}%
    \cslet{jmlr@thm@#1@postheader}{\@theorempostheader}%
    \jmlr@org@newtheorem{#1}%
  }
  \renewcommand*{\@xthm}[2]{%
    \def\@jmlr@currentthm{#1}%
    \@begintheorem{#2}{\csname the#1\endcsname}%
    \ignorespaces
  }
  \def\@ythm#1#2[#3]{%
    \def\@jmlr@currentthm{#1}%
    \@opargbegintheorem{#2}{\csname the#1\endcsname}{#3}%
    \ignorespaces
  }
  \renewcommand*{\@begintheorem}[2]{%
    \ifdef{\@jmlr@currentthm}%
    {%
      \letcs{\jmlr@this@theoremheader}{jmlr@thm@\@jmlr@currentthm @header@font}%
      \letcs{\jmlr@this@theorembody}{jmlr@thm@\@jmlr@currentthm @body@font}%
      \letcs{\jmlr@this@theoremsep}{jmlr@thm@\@jmlr@currentthm @sep}%
      \letcs{\jmlr@this@theorempostheader}%
         {jmlr@thm@\@jmlr@currentthm @postheader}%
    }%
    {%
      \let\jmlr@this@theorembody\@theorembodyfont
      \let\jmlr@this@theoremheader\@theoremheaderfont
      \let\jmlr@this@theoremsep\@theoremsep
      \let\jmlr@this@theorempostheader\@theorempostheader
    }%
    \trivlist
      \item
       [%
        \hskip\labelsep{\jmlr@this@theoremheader #1\ #2%
           \jmlr@this@theorempostheader}%
       ]%
      \mbox{}\jmlr@this@theoremsep
      \jmlr@this@theorembody
  }
  \renewcommand*{\@opargbegintheorem}[3]{%
    \ifdef{\@jmlr@currentthm}%
    {%
      \letcs{\jmlr@this@theoremheader}{jmlr@thm@\@jmlr@currentthm @header@font}%
      \letcs{\jmlr@this@theorembody}{jmlr@thm@\@jmlr@currentthm @body@font}%
      \letcs{\jmlr@this@theoremsep}{jmlr@thm@\@jmlr@currentthm @sep}%
      \letcs{\jmlr@this@theorempostheader}%
         {jmlr@thm@\@jmlr@currentthm @postheader}%
    }%
    {%
      \let\jmlr@this@theorembody\@theorembodyfont
      \let\jmlr@this@theoremheader\@theoremheaderfont
      \let\jmlr@this@theoremsep\@theoremsep
      \let\jmlr@this@theorempostheader\@theorempostheader
    }%
    \trivlist
     \item[\hskip\labelsep{\jmlr@this@theoremheader #1\ #2\ (#3)%
       \jmlr@this@theorempostheader}]%
     \mbox{}\jmlr@this@theoremsep
     \jmlr@this@theorembody
  }
\makeatother

\newtheorem{theorem}{Theorem}
\newtheorem{lemma}[theorem]{Lemma}
\newtheorem{proposition}[theorem]{Proposition}
\newtheorem{remark}[theorem]{Remark}
\newtheorem{corollary}[theorem]{Corollary}
\newtheorem{definition}[theorem]{Definition}

\newtheorem{assumption}[theorem]{Assumption}

\crefname{enumi}{Statement}{Statements} 
\crefname{equation}{}{}

\newcommand{\abs}[1]{\lvert #1 \rvert}
\newcommand{\norm}[1]{\| #1 \|} 
 
\newcommand{\innp}[1]{\langle #1 \rangle}

\DeclareMathOperator{\supp}{supp}                       

\renewcommand*\R{\mathbb{R}}                              

\let\epsilon\varepsilon
\newcommand{\defi}{\stackrel{\mathrm{\scriptscriptstyle def}}{=}}

\newcommand{\bigO}{O}

\newcommand{\cstLip}{G}
\newcommand{\cstSmooth}{L} 
\newcommand{\cstUnifC}{\mu} 
 
\newcommand{\cstRegC}{\alpha} 

\DeclareMathOperator*{\argmax}{arg\,max}                
\DeclareMathOperator*{\argmin}{arg\,min}                

\usepackage{lastpage}

\title[Black-Box Uniform Stability for Non-Euclidean Empirical Risk Minimization]{Black-Box Uniform Stability for Non-Euclidean Empirical Risk Minimization}

\coltauthor{%
 \Name{Simon Vary} \Email{\href{simon.vary@stats.ox.ac.uk}{simon.vary@stats.ox.ac.uk}}\\
       \addr Department of Statistics\\
       University of Oxford
       \AND
       \Name{David Martínez-Rubio} \Email{\href{dmrubio@ing.uc3m.es}{dmrubio@ing.uc3m.es}} \\
       \addr Signal Theory and Communications Department\\
       Carlos III University of Madrid
       \AND
       \Name{Patrick Rebeschini} \Email{\href{patrick.rebeschini@stats.ox.ac.uk}{patrick.rebeschini@stats.ox.ac.uk}} \\
       \addr Department of Statistics\\
       University of Oxford
}

\usepackage[natbib, backend=bibtex, maxcitenames=3, minalphanames=3, maxbibnames=99, style=alphabetic, hyperref, backref, useprefix=true, uniquename=false, doi=false,url=false,eprint=false]{biblatex}

\usepackage{csquotes}

\bibliography{refs} 

\newbibmacro{string+doiurlisbn}[1]{%
  \iffieldundef{doi}{%
    \iffieldundef{url}{%
      \iffieldundef{isbn}{%
        \iffieldundef{issn}{%
          #1%
        }{%
          \href{http://books.google.com/books?vid=ISSN\thefield{issn}}{#1}%
        }%
      }{%
        \href{http://books.google.com/books?vid=ISBN\thefield{isbn}}{#1}%
      }%
    }{%
      \href{\thefield{url}}{#1}%
    }%
  }{%
    \href{https://doi.org/\thefield{doi}}{#1}%
  }%
}

\DeclareFieldFormat{title}{\usebibmacro{string+doiurlisbn}{\mkbibemph{#1}}}
\DeclareFieldFormat[article,incollection,inproceedings]{title}%
    {\usebibmacro{string+doiurlisbn}{\mkbibquote{#1}}}

\usepackage{algorithm}
\usepackage{algcompatible}
\algnewcommand{\lst}{\texttt{lst}}
\algnewcommand{\slst}{\texttt{slst}}
\algnewcommand{\SEND}{\textbf{send}}

\newsavebox{\algleft}
\newsavebox{\algright}

\makeatletter
\newcounter{algorithmicH}
\let\oldalgorithmic\algorithmic
\renewcommand{\algorithmic}{%
  \stepcounter{algorithmicH}
  \oldalgorithmic}
\renewcommand{\theHALG@line}{ALG@line.\thealgorithmicH.\arabic{ALG@line}}
\makeatother



\input{definitions.tex}

\begin{document}

\maketitle

\begin{abstract}
	We study first-order algorithms that are uniformly stable for empirical risk minimization (\newtarget{def:acronym_empirical_risk_minimization}{\ERM{}}) problems that are convex and smooth with respect to $p$-norms, $p \geq 1$. We propose a black-box reduction method that, by  employing properties of uniformly convex regularizers, turns an optimization algorithm for Hölder smooth convex losses into a uniformly stable learning algorithm with optimal statistical risk bounds on the excess risk, up to a constant factor depending on $p$. Achieving a black-box reduction for uniform stability was posed as an open question by \citet{attia2022uniform}, which had solved the Euclidean case $p=2$. We explore applications that leverage non-Euclidean geometry in addressing binary classification problems.
\end{abstract}


\section{Introduction}
We study how to obtain an \emph{optimally} stable algorithm for a general learning problem from an optimization algorithm in a black-box manner. Given a distribution $P$ within a family $\mathcal{P}$, our task is to minimize the \emph{population risk} $f(x) = \mathbb{E}_{z\sim P} [\ell(x;z)]$ where $\ell(\cdot, z):\X\rightarrow \mathbb{R}$ is a convex smooth loss function for every $z$. In this work, we focus on convex smooth losses with respect to $p$-norms.  
The statistical question is how to bound the \emph{excess risk} of an estimator $\hat{x}$, which is measured as the difference between the population risk of $\hat{x}$ and the minimal population risk over the parameter domain $\X$:
\begin{equation*}
    \delta f(\hat x) = f(\hat x) - f(\tilde{x}),
\end{equation*}
where $\tilde{x} \in \argmin_{x\in\X} f(x)$. 

In the classical situation where the data distribution is unknown and we have access only to a finite training dataset $S=\left\{z_1, \ldots, z_n\right\}\subset\mathcal{Z}$ of $n$-samples drawn i.i.d.\ from $P$, we can use the \emph{empirical risk} $f_S(x) = \frac1n \sum_{i=1}^n \ell(x; z_i)$ as a sample-average proxy of the population risk.  The expected excess risk of an estimator is typically bounded by balancing the trade-off between error terms originating from the generalization component and those arising from offline optimization of the empirical risk:
\begin{equation}
     \mathbb{E}_S\left[\delta f(\hat x)\right] = \underbrace{\mathbb{E}_S\left[ f(\hat{x}) - f_S(\hat x) \right]}_{\text{generalization}} + \underbrace{\mathbb{E}_S\left[ f_S(\hat{x}) - f_S(\tilde{x}) \right]}_{\text{optimization}}. \label{eq:excess_risk_decomp1}
\end{equation}
Here the optimization error can be further upper bounded using the empirical risk minimizer (ERM) $x^\ast \in \argmin_{\X} f_S(x)$ and noting that $\mathbb{E}_S\left[ f_S(\hat{x}) - f_S(\tilde{x}) \right] \leq \mathbb{E}_S\left[ f_S(\hat{x}) - f_S(x^\ast) \right]$.

In this paper we are interested in providing optimal risk bounds which involve bounding both the generalization and the optimization error. The notion of \emph{algorithmic stability} \citep{bousquet2002stability, Shalev-Shwartz2009Stochastic} has been successfully used to control the generalization error. This notion pertains to the ability of the algorithm to be robust to small perturbations in the training dataset. 

Using the \emph{uniform stability} framework, which is a particular type of algorithmic stability, \citet{attia2022uniform} showed that it is possible to perform a black-box conversion from a given optimization algorithm for a convex, smooth objective to a uniformly stable algorithm with the same optimization convergence rate in the Euclidean case ($p=2$), up to a log factor. However, \citet{attia2022uniform} noted that it is currently unknown whether similar black-box results are achievable when the function regularity is measured with general non-Euclidean geometries $\norm{\cdot}_p$ for $p\geq 1$, posing the following open problem:
\begin{quote}
    \emph{Given an optimization algorithm for a convex, and smooth objective in $\ell_p$-geometry for $p\geq 1$, is it possible to perform a black-box conversion to an algorithm with the same convergence rate that is also uniformly stable?}
\end{quote}

\subsection{Contributions}

In this paper, motivated by the open problem of \citet{attia2022uniform}, we develop techniques for uniformly stable empirical risk minimization with convex smooth objectives with respect to non-Euclidean norms $\norm{\cdot}_p$. By studying the uniform stability of estimators that are approximate minimizers of the empirical risk minimization after adding uniformly convex regularization, see~\cref{lemma:stability_unif_reg}, we design uniformly stable algorithms for problems with regularity measured in general $\norm{\cdot}_p$ normed spaces.

\paragraph{Black-box reduction in non-Euclidean geometry.} We design an algorithm, \emph{Uniform Stable Optimization with $\ell_p$ Regularity} ($\mathrm{USOLP}$), see \cref{alg:usolp}. The algorithm performs a black-box reduction from a given optimization algorithm for convex Hölder smooth functions that is assumed to remain within a bounded domain to a learning algorithm with the same convergence rate that is \emph{optimally uniformly stable} (\cref{thm:blackbox_uniform_stabp}). Our black-box reduction to a stable algorithm is based on: (A) stability properties of uniformly convex regularization (\cref{lemma:stability_unif_reg}), and (B) observing that restarting an optimization algorithm for a convex Hölder smooth functions captures the \emph{uniformly convex structure} of the objective and achieves corresponding convergence rates, which can be faster in some situations, e.g., see \citep{renegar2022simple}. Thus we provide a positive answer to the open problem posed by \citet{attia2022uniform} with the caveat that our USOLP reduction algorithm for $\ell_p$ geometry works for optimization methods for $(\cstSmooth,\min\{2,p\})$-Hölder smooth objectives and requires knowledge of an upper bound on the distance of the initial point to the minimizer, the Lipschitz constant, and the guarantee that the given optimization method remains in the corresponding domain. The reduction is also optimal, since it exhibits the best convergence versus stability trade-off: an improvement to either of the rates would contradict the statistical lower bounds that were derived by \citet{levy2019necessary} for $p\in[1,2]$, and that we establish for $p\geq 2$ in \cref{thm:risk_lower}.

\paragraph{Optimal expected excess risk bounds.}

\cref{thm:risk_upper} shows that an (approximate) minimizer of the empirical risk with a specific level of uniformly convex regularization achieves an optimal expected excess risk, up to a constant factor. Consequently, \cref{corr:usolp_excess_risk_low_dim} and \cref{corr:usolp_excess_risk_high_dim} show that our black-box reductions converge to a point $x_T$ that achieves an excess risk
\begin{align*}
    \mathbb{E}_S\left[ \delta f(x_T)\right] = \begin{cases}
        \bigO_p\left(\cstSmooth R^2 \frac{d^{1/2-1/\hat{p}}}{n^{1/2}} \right) \quad & \text{when}\, d\leq n \\
        \bigO_p\left(\cstSmooth R^2 \frac{1}{n^{1/\hat{p}}} \right) \quad & \text{when}\, d>n
    \end{cases},
\end{align*}
that is optimal up to a constant factor in $p$ for $\cstSmooth$-smooth losses over the domain $\ballp[x_0][R]$ where $x_0$ is the starting point of our algorithm and $\hat{p} = \max\{2,p\}$. These rates are identical up to a constant to the excess risk bound obtained by employing the \emph{non}-black-box algorithm proposed by \citep{diakonikolas2021complementary}, see \cref{corr:usolp_excess_risk_high_dim} and \cref{thm:diakonikolas_generalization_bound} respectively. We summarize our results in \cref{tab:overview}.

\begin{table*}[t]
\resizebox{\textwidth}{!}{%
\begin{tabular}{c | c c c}
    \toprule
    & LB (Non-Euclidean, $\ell_p$-norm) & UB (Euclidean $\ell_2$-norm) & UB (Non-Euclidean, $\ell_p$-norm)\\ \midrule
    $d \leq n$ & $\widetilde\Omega_p(\cstSmooth R^2 \frac{d^{1/2-1/\hat{p}}}{n^{1/2}})$ \citep{levy2019necessary} & \multirow{2}{4cm}{\quad  \ \ \ $\widetilde\bigO_p(\cstSmooth R^2 \left(\frac{1}{n}\right)^{1/2})$  \\ \quad \quad \quad \ \  \citep{attia2022uniform} }
    & $\widetilde\bigO_p(\cstSmooth R^2 \frac{d^{1/2-1/\hat{p}}}{n^{1/2}})$  (\cref{corr:usolp_excess_risk_low_dim}) \\
    $d > n$ & $\widetilde\Omega_p(\cstSmooth R^2 \left(\frac{1}{n}\right)^{1/\hat{p}})$ (\cref{thm:risk_lower}) &  & $\widetilde\bigO_p(\cstSmooth R^2 \left(\frac{1}{n}\right)^{1/\hat{p}})$ (\cref{corr:usolp_excess_risk_high_dim})\\
    \bottomrule
\end{tabular}}
\caption{Excess risk bounds for black-box reduction algorithms for \ERM{} with loss functions in $\R^d$ that are $\cstSmooth$-smooth over the ball of radius $R$ w.r.t.\ $\ell_p$-norm for $n=\Omega_p(1)$. We denote $\hat{p} = \max\{ p, 2\}$. The lower bound on the excess risk (LB) is for estimators within an $\ell_p$-ball of radius $R$ and by observing that the Lipschitz constant is bounded by $2\cstSmooth R$. The lower bound that we establish in the case $d>n$ improves upon the lower bound derived in \citep{sridharan2012Learning} as discussed in \cref{sec:related_work}.} The upper bounds on the excess risk (UB) are achieved by the black-box reduction USOL2 \citep{attia2022uniform} when $p=2$ and USOLP (\cref{alg:usolp}) for $p\geq1$.\label{tab:overview}
\end{table*}

\subsection{Related work}\label{sec:related_work}

The first results obtaining generalization bounds via algorithmic stability can be found in \citep{rogers1978finite,devroye1979distribution}. The work of \citet{bousquet2002stability} continued this direction and provided guarantees for general supervised learning algorithms with regularization. \citet{hardt2016train} showed algorithmic stability for stochastic gradient descent without explicit regularization by exploiting gradient descent's non expansivity and \citet{mou2018generalizations} showed stability for the stochastic gradient Langevin algorithm. Several works have focused on obtaining algorithmic stability for smooth convex functions, where the dependence on number of iterations for the optimization rates is faster than for the class of Lipschitz convex functions. In particular, \citet{chen2018stability} showed that accelerated gradient descent \citep{nesterov1983method} enjoys certain uniform stability for quadratic functions, pointed out the trade-off between stability and optimization showing as a result a lower bound on the optimization rates for this algorithm of $\bigomega{1/T^2}$. \citet{attia2021algorithmic} showed that for general convex smooth functions the uniform stability grows exponentially with the number of iterations. \citet{attia2022uniform} designed a method that obtains a stable algorithm from an optimization algorithm for Euclidean smooth convex functions by using the latter as a black-box. They also show stability for a specific unaccelerated algorithm designed for convex smooth functions with respect to $\norm{\cdot}_p$, for $p\in[1,2]$. They pose an open problem, restated in the introduction, whose solution implies that it is possible to use optimization methods for convex and smooth functions with $\ell_p$-norm regularity to achieve statistical guarantees. In this work, we solve the open problem in an affirmative with the additional benefit of recovering \emph{optimal} statistical guarantees. \citet{zhang2022bring} use Euclidean regularization in order to obtain a procedure that uses an optimization algorithm as a black-box and obtains a differentially private algorithm.

Concerning optimization, \citep{nemirovskii1985optimal,khachiyan1993optimal} presented modifications of Nesterov's accelerated gradient descent in order to obtain accelerated algorithms for convex and smooth function with respect to $p$-norms. An algorithm and proof can also be found in \citep[Section 4]{daspremont2018optimal}. Also for non-Euclidean regularity assumptions, \cite{diakonikolas2021complementary} provided algorithms for objectives consisting of the sum of a smooth term and a proximable term, with a special focus on achieving acceleration when the proximable term is uniformly convex. The studied stability 
for $p\in[1,2]$ with their specific algorithm via regularization, as opposed to our black-box algorithm transformation. \citet{juditsky2014deterministic} studied optimization algorithms for Lipschitz, uniformly convex problems with respect to non-Euclidean norms.

We require for the algorithms that we transform that they either work with a compact convex set as a constraint or that their iterates are guaranteed to naturally stay in one such set. A variety of unconstrained algorithms for convex problems have the property that their iterates remain in a ball of center a minimizer $x^\ast$ and radius a constant times the initial distance to $x^\ast$. For instance, this is satisfied by gradient descent for convex smooth functions with respect to $\norm{\cdot}_2$. Similarly, it is possible to show that the iterates of Nesterov's accelerated gradient descent, as well as their more general accelerated hybrid proximal extragradient algorithms, remain bounded in a ball around a minimizer of radius $\bigo{R_0}$, where $R_0$ is the initial distance to this minimizer \citep[Theorem 3.10]{monteiro2013accelerated}. The iterates $x_t$ of mirror descent algorithms with a $(\cstUnifC, \nu)$-uniformly convex regularizer $\psi$ with respect to a norm $\norm{\cdot}$, satisfy $\frac{\cstUnifC}{\nu}\norm{x^\ast - x_t}^\nu \leq D_\psi(x^\ast, x_t) \leq D_{\psi}(x^\ast, x_0)$. 

There are known lower bounds for the minimax risk for a class of random loss functions in $\ballpzero[1]$ which are linear and $1$-Lipschitz w.r.t.\ $\ell_p$-norm. For the low-dimensional case, i.e., $d\leq n$, and $p\geq 1$, \citet{levy2019necessary} provided lower bounds for the minimax risk which are $\widetilde\Omega_p(d^{1/2-1/\hat{p}} / n^{1/2})$ where $\hat{p}=\max\{p,2\}$ and they showed that a mirror descent with a specific distance function achieves these rates. In the high-dimensional case when $d > n$, the work of \citet[Section 8]{sridharan2012Learning} combined with the fact that the $\ell_{p^\ast}$ space cannot have Rademacher type greater than $p^\ast = p/(p-1)$ when $p^\ast \in [1, 2]$ \citep[Remark 6.2.11.g]{albiac2006topics}, implies that the lowest minimax risk for a class of random linear $1$-Lipschit losses in $\ballpzero[1]$ cannot be less or equal to $c_p n^{-1/p-\delta}$ when $p\geq 2$ for any universal constants $c_p, \delta >0$ where $c_p$ may depend on $p$. We show a slightly stronger of $\Omega_p(1/n^{1/\hat{p}})$ for $\hat{p}=\max\{p,2\}$ in \cref{thm:risk_lower} using a very different technique, inspired to the construction given by \citet{levy2019necessary}. We match these rates with our black-box reduction algorithm as shown in \cref{thm:blackbox_uniform_stabp} for $p\in[1,2]$.

\paragraph{Notation and terminology.} 

We denote $g^\ast(y)= \sup_{x\in\mathbb{R}^d} \{\innp{x,y} - g(x) \}$ the Fenchel dual function of a function $g$.  We denote a general norm as $\norm{\cdot}$ and its dual norm as $\norm{\cdot}_\ast$. The $\ell_p$-norm is denoted as $\norm{\cdot}_p$. We use $\newtarget{def:ball}{\ballp[x][R]}$ for the $\ell_p$ ball of center $x$ and radius $R$, and we use $\ballpzero[R]$ if it is centered at $0$. As it is usual, we say $\hat{x}$ is an $\epsilon$-minimizer of the problem $\min_{x\in\X}f(x)$ if $f(\hat{x})- \min_{x\in\X}f(x) \leq \epsilon$.

\section{Preliminaries}

\subsection{Generalizations of Smoothness and Convexity}
In this work we are interested in the optimization of convex functions that are smooth with respect to non-Euclidean geometries, over compact convex sets $\newtarget{def:optimization_domain}{\X}\subseteq \mathbb{R}^d$, i.e., when measuring distances with $p$-norms. The standard notions of strong convexity and smoothness are not well suited for this setting.  For example, function $\psi(x) = \frac12 \|x\|^2_p$ has its strong convexity parameter bounded above by $1/d^{1-\frac2p}$ \citep{diakonikolas2021complementary} which is small for high-dimensional problems and yields dimension-dependent rates when $p\gg2$. The same is true for any strongly convex function w.r.t\ $\norm{\cdot}_p$ whose range is bounded above by a constant on a unit $\ell_p$-ball \citep[Example 5.1]{daspremont2018optimal}.

In order to get dimension-independent rates, one can exploit instead the regularity of the objective function in terms of uniform convexity and Hölder smoothness, which are generalizations of strong convexity and smoothness.

\begin{definition}[Uniform convexity] \label{def:uniform_convexity}
We say $f$ is $(\cstUnifC, \nu)$-uniformly convex with respect to norm $\norm{\cdot}$ if for any $x, y \in \R^d$ we have
    \begin{equation*}
        f\left(t x + (1-t)y \right) \geq t f(x) + (1-t) f(y) + t(1-t)\frac{\cstUnifC}{\nu}\| x-y \|^\nu,
    \end{equation*}
    or, equivalently, for differentiable $f$:
    \begin{equation*}
        f(y) \geq f(x)-\innp{\nabla f(x), x - y} + \frac{\cstUnifC}{\nu} \| x - y\|^{\nu}, \forall  x, y \in \R^d,
    \end{equation*}
    which is shown, for example in \citep[Section 4.2.2]{nesterov2018lectures}.
\end{definition}
Note that the above definition simplifies to the notion of strong convexity when $\nu=2$.

\begin{definition}[Hölder smoothness]\label{def:holder_smooth}
    We say that  $f:\mathbb{R}^d\rightarrow \mathbb{R}$ is $(\cstSmooth, \nu)$-Hölder smooth w.r.t.\ $\norm{\cdot}$, if there exists $\cstSmooth>0$ and $\nu\in[1,2]$, such that for all $x, y \in \R^d$
    \begin{equation*}
        f(y) \leq f(x) + \innp{\nabla f(x), y-x} + \frac{\cstSmooth}{\nu} \norm{x - y}^{\nu},
    \end{equation*}
    or equivalently, when $f$ is also convex,
    \begin{equation*}
        \norm{\nabla f(x) - \nabla f(y)}_\ast \leq \cstSmooth \norm{x - y}^{\nu-1} \text{ for all } x,y\in\R^d.
    \end{equation*} 
\end{definition}
We say that a function is $\cstSmooth$-smooth if it is $(\cstSmooth, 2)$-Hölder smooth. 

The notion of uniform convexity allows us to capture function regularity in terms of non-Euclidean  $\ell_p$-norm geometry without needing to resort to dimension-dependent bounds. The function of $\ell_p$-norm squared $\psi(x) = \frac12\norm{x}_p^2$ has its strong convexity constant bounded above by $1/d^{1-2/p}$, for $p\geq 2$ \citep{diakonikolas2021complementary}, which becomes small in high-dimensional settings. In fact, \citet[Example 5.1]{daspremont2018optimal} show that \emph{any} strongly convex function w.r.t.\ $\norm{\cdot}_p$ with a
constant bounded range in a unit $\ell_p$-ball must have its strong convexity smaller than $1/d^{1-2/p}$ for $p\geq 2$. The notion of uniform convexity (\cref{def:uniform_convexity}) allows to show that $\frac1p\norm{x}_p^p$ is $(2^{2-p}, p)$-uniformly convex w.r.t\ $\norm{\cdot}_p$, i.e., its uniform convexity constant is \emph{dimension-independent}, see \cref{lemma:unif_cvxty_of_norm_p_to_the_p} and references therein. \citet{ball1994sharp} show that the bound on the uniform convexity is no longer dimension-dependent with the aforementioned $\psi(x) = \frac1p\norm{x}_p^p$ being a specific example of a function whose uniform convexity constant is dimension-independent.

The fact that the uniform convexity constant of $\psi(x) = \frac1p\norm{x}_p^p$ w.r.t.\ $\norm{\cdot}_p$ norm is dimension-independent combined with optimization algorithms for uniformly convex objectives with dimension-independent convergence rates, e.g., Generalized AGD+ \citep{diakonikolas2021complementary}, allows us to derive learning algorithms with both, dimension-independent stability and optimization rates.

\begin{lemma}[Duality of Hölder smoothness and uniform convexity]\linktoproof{lemma:duality_holder_unifconvex} \label{lemma:duality_holder_unifconvex}
    Let $f:\R^{n}\rightarrow\R$ be $(\cstUnifC, \nu)$-uniformly convex function w.r.t.\ $\|\cdot\|$. Then its Fenchel dual $f^*$ is $((\frac{\nu^\ast}{2\cstUnifC(\nu^\ast-1)})^{\nu^\ast-1}, \nu^\ast)$ - Hölder smooth w.r.t.\ $\|\cdot\|_*$, where $\nu^\ast = \nu/(\nu-1)$.
\end{lemma}
The proof is given in \cref{subsec:proof_duality_holder_unifconvex} and follows a similar reasoning as the proof of \citep[Lemma 15]{shalev2007online} which shows duality between strongly convex and smooth functions. A different proof can be found also in \citep[Theorem 3.5.5]{zalinescu2002convex}.

\subsection{Algorithmic stability}
There are different notions of algorithmic stability that are useful for controlling generalization error \citep{bousquet2002stability, kutin2013almost}. Herein, we use the notion of \emph{uniform algorithmic stability} introduced by \citet{bousquet2002stability}.

\begin{definition}[Uniform stability]\label{def:stability}
    An algorithm $\mathcal{A}$, which outputs an estimator $\hat x$ and $\hat x'$ when given the datasets $S$ and $S'$ respectively, that differ in at most one sample, is said to be \emph{$\newtarget{def:stability}{\stability}$-uniformly stable}, if
    \begin{equation*}
        \sup_{z\in\mathcal{Z}} \left| \ell(\hat x; z) - \ell(\hat x'; z) \right| \leq \stability(\mathcal{A}).
    \end{equation*}
\end{definition}

A known result of \citet{bousquet2002stability} is that an estimator $\hat{x}$ obtained from a uniformly stable algorithm $\mathcal{A}$ has bounded the generalization error term in \eqref{eq:excess_risk_decomp1} as 
\begin{equation}\label{eq:stability_to_generalization}
    \mathbb{E}_S\left[ f(\hat{x}) - f_S(\hat x) \right] \leq \stability(\mathcal{A}). 
\end{equation}

\section{Stability after uniformly convex regularization}\label{sec:unif_convex_generalization}

We study the uniform stability of estimators that approximately solve the regularized problem
\begin{equation}
    x^*_\cstUnifC \in \argmin_{x\in\X} f^{(\cstUnifC)}_S(x) \defi f_S(x) + \cstUnifC\psi(x),\label{eq:optimization_regul}
\end{equation}
where $\psi(x):x\in\R^d\rightarrow \R$ is a $(1,\nu)$-uniformly convex regularizer that attains its minimum at a point $x_0\in\R^d$. For example, a viable choice for a $(1,p)$-uniformly convex regularizer based on the $\ell_p$-norm, which is also used in \cref{sec:blackbox_stability}, is $\psi(x)=\frac{2^{p-2}}{p}\norm{x- x_0}_p^p$ for $p\geq 2$ and some $x_0\in\X$, cf. \cref{lemma:unif_cvxty_of_norm_p_to_the_p}. Our results assume an upper bound on the regularizer distance between the minimum of the \ERM{} and the regularizer, i.e., $\psi(x^\ast) - \psi(x_0) \leq D^\nu$ for $x^\ast\in\argmin_{x\in\mathcal{X}} f_S(x)$, for $D\geq 0$.

We show that for the right choice of  $\cstUnifC$, the approximate solutions $\hat{x}$ to the optimization problem in \eqref{eq:optimization_regul} are uniformly stable while also having bounded distance to the original solution of the \emph{unregularized} empirical risk.

\begin{lemma}[Stability after uniformly convex regularization] \linktoproof{lemma:stability_unif_reg}\label{lemma:stability_unif_reg}
    Let $\nu\geq 2$, $\cstUnifC>0$, and $\mathcal{A}$ be an algorithm that for a dataset $S$ returns an $\hat\varepsilon$-minimizer $\hat{x}$ of the regularized \ERM{} $f_S^{(\cstUnifC)}$, where the losses $\ell(\cdot,z)$ are $\cstLip$-Lipschitz and the added regularizer $\cstUnifC\psi(x)$ is $(\cstUnifC,\nu)$-uniformly convex  w.r.t.\ $\norm{\cdot}$. If the accuracy is bounded as $\hat\varepsilon \leq \min \{ \cstUnifC D^\nu, (\frac{\nu}{\cstUnifC})^{1/(\nu-1)} (\frac{2\cstLip}{n})^{\nu/(\nu-1)} \}$, then the algorithm $\mathcal{A}$ is uniformly stable as
    \begin{equation*}
        \stability(\mathcal{A}) \leq 3 \left( \frac{2\nu}{n\cstUnifC} \cstLip^\nu \right)^{1/(\nu-1)},
    \end{equation*}
    while the optimization error of $\hat{x} = \mathcal{A}(S)$ on the unregularized empirical risk is upper bounded as 
    \begin{equation*}
        \mathcal{E}_\mathrm{opt} \defi  f_S(\hat{x}) - \min_{x\in\X} f_S(x) \leq 2\cstUnifC D^\nu.
    \end{equation*}
\end{lemma}
The proof is given in \cref{subsec:proof_stability_unif_reg} and follows a similar technique as the proof of \citep[Lemma 7]{attia2022uniform}, but differs by using properties of uniformly instead of strongly convex regularizers and by considering only approximate minimizers of the regularized \ERM{}.

Recall that the expected excess risk is bounded by the sum of the optimization error and the stability error. By \eqref{eq:stability_to_generalization}, which is due to \citet{bousquet2002stability}, we have that stability controls the generalization error. In the result that follows, we find the optimal choice of $\cstUnifC$ and the necessary precision for the optimization error on $f_S^{(\cstUnifC)}$, given $n, \nu, \cstLip$, and $D$, for which the upper bound on the expected excess risk is minimized.

\begin{theorem}[Excess risk bound after uniformly convex regularization]\linktoproof{thm:risk_upper}\label{thm:risk_upper}
    Let $\nu\geq 2$ and $\hat{x}_\cstUnifC$ be a $(6 \cstLip D / n^{1+1/\nu})$-minimizer of the regularized \ERM{} $f^{(\cstUnifC)}_S(x)$ in \eqref{eq:optimization_regul} with $\cstLip$-Lipschitz losses $\ell(\cdot,z)$ w.r.t.\ $\norm{\cdot}$, for a constant $\cstUnifC = \bigO_\nu(D^{1-\nu}\cstLip/n^{1/\nu})$. Then the expected excess risk of $\hat{x}_\cstUnifC$ is upper bounded as
    \begin{equation*}
        \mathbb{E}_S \left[ \delta f(\hat{x}_\cstUnifC)\right] \leq 8 \cstLip D \left( \frac{1}{n} \right)^{1/\nu}.
    \end{equation*}
\end{theorem}
The proof is in \cref{subsec:proof_risk_upper} and consists of bounding the generalization and optimization errors using the uniform stability bound derived in \cref{lemma:stability_unif_reg}, finding the optimal choice of $\cstUnifC$ that minimizes the excess risk, and ensuring that the error of  the estimate $\hat x_\cstUnifC$ is sufficiently small. 

Note that the results of \cref{lemma:stability_unif_reg} and \cref{thm:risk_upper} study properties of an approximate minimizer of the regularized \ERM{} and are not concerned by how the approximate minimizer is computationally obtained. In cases when the bound on the regularizer implies a bounded domain containing the global minimizer, e.g., when $\X = \mathcal{B}_{\norm{\cdot}}(x_0, R)$ for $R>0$ and $x^\ast\in\X$, the smoothness of $f$ implies its Lipschitzness
\begin{equation}
    \norm{\nabla f(x)} = \norm{\nabla f(x) \!-\! \nabla f(x^\ast)} \leq \cstSmooth \norm{x-x^\ast} \leq 2 \cstSmooth R,\!\! \label{eq:smooth_to_lipschitz}
\end{equation}
where we used that $\nabla f(x^\ast) = 0$. As a result, \cref{lemma:stability_unif_reg} and \cref{thm:risk_upper} apply also when the loss function has only prescribed smoothness inside of a bounded domain containing the global minimizer since $\cstLip \leq 2 \cstSmooth R$.

\subsection{Lower bound on the expected excess risk}

In the following theorem we extend the technique of \citep[Proposition 1]{levy2019necessary}
to the high-dimensional case using sparse distributions and establish a lower bound on the expected excess risk for the specific case when the regularity of the loss $\ell(\cdot,z)$ is in respect to the $\ell_p$-norm and $x^\ast\in\ballpzero[R]$ for $R>0$.

\begin{theorem}\linktoproof{thm:risk_lower}\label{thm:risk_lower}
    For $p \geq 1$, $48n < d$, and $\ell(x;z)$ that is linear and $\cstLip$-Lipschitz w.r.t.\ $\norm{\cdot}_p$, we have that the expected excess risk in $\ballpzero[R]$ is lower bounded as
    \begin{equation*}
        \inf_{\hat{x}\in\ballpzero[R]} \sup_{P\in\mathcal{P}}\mathbb{E}_{S\sim P^n}  \left[ \delta f(\hat x)\right] \geq \frac{\cstLip R}{12} \min\left\{1, \sqrt{\dfrac{\log(d)}{n}}\right\} \\
    \end{equation*}
    when $1 \leq p \leq 1+1/\log(d)$ and
    \begin{equation*}
        \inf_{\hat{x}\in\ballpzero[R]} \sup_{P\in\mathcal{P}}\mathbb{E}_{S\sim P^n}  \left[ \delta f(\hat x)\right] \geq 
            \frac{\cstLip R}{16} \left(\dfrac{1}{n}\right)^{1/\hat{p}}
    \end{equation*}
    when $p > 1 + 1/\log(d)$ 
    where $\hat{p} = \max\{p,2\}$.
\end{theorem}

The proof can be found in \cref{subsec:proof_risk_lower}. The theorem shows that there exists a family of probability distributions and losses for which the performance of the best estimator $\hat{x}\in\ballpzero[R]$ in terms of the excess risk cannot be lower than a certain bound. This result improves over the one implied by \citet[Section 8]{sridharan2012Learning} as explained in \cref{sec:related_work}. From \citep[Proposition 1]{levy2019necessary}, we have that the risk is lower bounded as $\widetilde\Omega(\cstLip R n^{-1/2})$ when $d \leq 48n$, where $\cstLip$ and $R$ are w.r.t.\ the $\ell_p$-norm.

Since any linear loss function in \cref{thm:risk_lower} is also convex, the lower bound also applies to problems defined for convex $G$-Lipschitz losses, denoted as $\mathcal{G}(\cstLip)$:
\begin{align*}
    \frac{1}{16}\cstLip R \left(\frac{1}{n}\right)^{1/\hat{p}} &\leq 
     \inf_{\hat{x}\in\ballpzero[R]} \sup_{\substack{P\in\mathcal{P} \\ \ell(\cdot,z)\in \mathcal{G}(\cstLip)}}\mathbb{E}_{S\sim P^n}  \left[ \delta f(\hat x)\right].
\end{align*}

By \cref{thm:risk_upper}, we obtain a matching upper bound on the excess risk when $p\geq 2$ using $\psi(x) = (2^{p-2}/p)\norm{x}_p^p$ as the regularizer, which is $(1,p)$-uniformly convex by \cref{lemma:unif_cvxty_of_norm_p_to_the_p}. From the lemma, we get that the excess risk of $\hat{x}$ that is a $(12\cstLip R/n^{1+\frac1p})$-minimizer and satisfies
\begin{align*}
     \sup_{\substack{P\in\mathcal{P} \\ \ell(\cdot,z)\in \mathcal{G}(\cstLip)}} \mathbb{E}_{S\sim P^n} \left[ \delta f(\hat{x})\right]&\leq 16 \cstLip R \left( \frac{1}{n} \right)^{1/p},
\end{align*}
where $R$ denotes the radius of the domain $\ballp[x_0][R]$ containing $x^\ast$, which comes from the form of $\psi(x)$ and its bounded range as $D = 2^{1-1/p}p^{-1/p}R\leq 2R$.

As a consequence, the estimator $\hat{x}$ gives an optimal bound up a constant in $p$ on the excess risk for \ERM{} with $\cstLip$-Lipschitz convex losses
\begin{align*}
     \frac{1}{16} \cstLip R \left( \frac{1}{n} \right)^{1/p} &\leq \inf_{\hat{x}\in\ballpzero[R]} \sup_{\substack{P\in\mathcal{P} \\ \ell(\cdot,z)\in \mathcal{G}(\cstLip)}}\mathbb{E}_{S\sim P^n}  \left[ \delta f(\hat x)\right] \\
     &\leq 16  \cstLip R \left( \frac{1}{n} \right)^{1/p},
\end{align*}
when $d> 48 n$. The optimality of the rate does not depend on the smoothness of $\ell(\cdot,z)$ since the upper bound from \cref{thm:risk_upper} applies regardless of how smooth the loss is and the lower bound applies for linear losses that are $0$-smooth.

\section{Non-Euclidean Black-box uniform stability}
\label{sec:blackbox_stability}

We show that applying an optimization algorithm to the \ERM{} problem with the added regularization term based on $\norm{\cdot}_p^p$ combined with a restart scheme yields a black-box reduction to an \emph{optimally} stable learning algorithm with the same convergence rate on the optimization error. Let
\begin{equation}
    x^*_\cstUnifC \in \argmin_{x\in\X} f^{(\cstUnifC)}_S(x) \defi f_S(x) +  \cstUnifC\frac{\cstRegC}{p} \norm{x - x_0}^p_p,\label{eq:optimization_regul_usolp}
\end{equation}
where $\cstUnifC>0$ and 
\begin{equation*}
    \cstRegC = \begin{cases}
        2^{p-2} &\quad \text{for $p\geq 2$}\\
        2^{2p-3} \left( 1 - \frac{1}{p}\right)^{p-1} &\quad \text{for $p \in (1,2)$}
    \end{cases}
\end{equation*}
ensures that $\psi(x)=\frac{\cstRegC}{p}\norm{x-x_0}_p^p$ is $(1,p)$-uniformly convex for $p\geq 2$ and $(1,p)$-Hölder smooth for $p\in(1,2)$ w.r.t.\ $\ell_p$-norm, see \cref{lemma:unif_cvxty_of_norm_p_to_the_p}. We will also use that $\psi(x)$ is \emph{locally}  smooth for $p\geq 2$ and strongly convex for $p\in(1,2)$ w.r.t.\ $\ell_p$-norm, see \cref{lemma:p_norm_ball_smooth_str_convex}.

The black-box reduction algorithm USOLP, given in \cref{alg:usolp}, transforms an optimization algorithm $\mathcal{A}$ meant for convex Hölder smooth objective functions to an optimally stable algorithm by solving \eqref{eq:optimization_regul_usolp} with an appropriate chosen $\cstUnifC>0$. Adding the regularization term $\psi(x)$ comes with two benefits: (i) statistically, it guarantees the minimizer is uniformly stable due to \cref{lemma:stability_unif_reg}, and (ii) in optimization, it makes the objective \emph{uniformly convex} and smooth when $p\geq 2$ or convex and \emph{Hölder smooth} when $p\in(1,2)$, which with an appropriate restarting scheme yields appropriate convergence rate that in some situations is faster compared to the original algorithm $\mathcal{A}$.

Let $\mathcal{A}(f_S^{(\cstUnifC)}, x_0, R, \hat\varepsilon)$ be an optimization algorithm for convex Hölder smooth objective functions that takes an initial point $x_0$, an upper bound $R\geq 0$ on the distance of $x_0$ to the minimizer, and a required target accuracy $\hat\varepsilon$. The algorithm then outputs a point $\hat{x} \in \ballp[x_0][R]$ such that 
\begin{equation}\label{eq:eps_minimizer}
    f_S^{(\cstUnifC)}(\hat{x}) - \inf_{x\in\ballp[x_0][R]} f_S^{(\cstUnifC)}(x) \leq \hat\varepsilon,
\end{equation}
where to achieve the goal the algorithm performs at most $\hat{T}$ gradient oracle calls.

The assumption, that the optimization algorithm outputs a point in $\ballp[x_0][R]$ is formalized below and is used in \cref{lemma:p_norm_ball_smooth_str_convex} to control smoothness and strong convexity of $\norm{\cdot}_p^p$ when $p\geq 2$ and $p\in(1,2)$ respectively.
\begin{assumption}\label{ass:within_p_ball}
    For any initial point $x_0$, and for some $R>0$, the iterates and the output of Algorithm $\mathcal{A}$ optimizing $f$ with a minimizer $x^\ast$ are in the domain $\mathcal{X} = \ballp[x_0][R]$ containing $x^\ast$.
\end{assumption}
We note that the two main cases where this assumption is satisfied is for (i) algorithms that work when constrained to $\X$, or for (ii) unconstrained algorithms whose iterates naturally stay in a ball around the minimizer $\ballp[x^\ast][2 R]$, for $R  =\bigO(\norm{x_0-x^\ast})$. As explained in \cref{sec:related_work}, many known algorithms fall in the second category. Note that \cref{ass:within_p_ball} guarantees also the condition on the bounded range $\psi(x^\ast) - \psi(x_0)\leq D^p$ in \cref{lemma:stability_unif_reg}, which for $\psi(x)=(\cstRegC/p)\norm{x-x_0}_p^p$ translates into the requirement that $x^\ast\in\mathcal{B}_{\norm{\cdot}_p}(x_0,R)$ with $R = (p/\cstRegC)^{1/p}D$. 

Since we study the stability of first order methods for smooth objectives whose convergence rates are based on the smoothness of the objective function, in this section we pose the results for $\cstSmooth$-smooth loss functions using the fact that $\cstLip \leq 2 \cstSmooth R$ on a bounded domain containing the minimizer as given in \eqref{eq:smooth_to_lipschitz}.

\begin{algorithm}[t]
\caption{$\mathrm{USOLP}(\mathcal{A})$ for $\ell_p$-structured learning}\label{alg:usolp}
\begin{algorithmic}
    \REQUIRE{Algorithm $\mathcal{A}$, estimate $R$, $x_0\in\R^d$, $T=\widetilde\Omega_p(n^{1/\gamma})$, $p\geq1$}
    \STATE{$\hat{p}_1 = \min\{ p, 2 \}$}
    \STATE{$\hat{p}_2 = \max\{ p, 2 \}$}
    \IF{$p \geq 2$}
        \STATE{$\cstUnifC  = (4C/T)^\gamma p \cstRegC^{2/p-1} \cstSmooth R^{2-p} $}
        \STATE{$r = \lceil \log_2( (T/C)^{\gamma}\cstRegC^{-2/p})\rceil $}
    \ELSIF{$p \in (1,2)$}
        \STATE{$\cstUnifC  = (C/T)^\gamma \frac{4}{p - 1} \cstRegC^{-2/p}\cstSmooth R^{2-p}$}
        \STATE{$r = \lceil \log_2( R^{p-2} (T/C)^{\gamma}/(2\cstRegC))\rceil $}
    \ENDIF
    \STATE{$R_0 = R$}
    \STATE{$\hat\varepsilon_0 = \cstSmooth R_0^{\hat{p}_1} / \hat{p_1}$}
    \FOR{ $i = 1, \ldots, r$ }
        \STATE {$\hat\varepsilon_{i}=\hat\varepsilon_{i-1} / 2$}
        \STATE {$x_{i} = \mathcal{A}(f_S^{(\cstUnifC)}, x_{i-1}, R_{i-1}, \hat\varepsilon_i)$}
        \STATE {$R_{i} =  \left( \frac{\hat{p}_2 R^{2 - \hat{p}_1} \hat\varepsilon_i}{\cstUnifC (\hat{p}_1 - 1)}\right)^{1/\hat{p}_2} $}
    \ENDFOR
    \ENSURE{$x_T \defi x_r$}
\end{algorithmic}
\end{algorithm}

\begin{theorem}[Black-box uniform stability] \linktoproof{thm:blackbox_uniform_stabp}\label{thm:blackbox_uniform_stabp}
    Let $p>1$, the loss function $\ell:\mathbb{R}^d\rightarrow\mathbb{R}$ be convex and $\cstSmooth$-smooth w.r.t. $\norm{\cdot}_p$ and $\mathcal{A}$ be an optimization algorithm satisfying \cref{ass:within_p_ball} with $R>0$ that has a convergence rate $C\hat{\cstSmooth} \|x_0 - x^*\|_p^{\hat{p}_1}/ \hat{T}^\gamma$ for convex, $(\hat{\cstSmooth},\hat{p}_1)$-Hölder smooth objective functions,  where $\hat{p}_1 = \min \{p , 2 \}$. Then, for $T = \widetilde\Omega_p(n^{1/\gamma})$, the iterate $x_T$ produced by $\mathrm{USOLP}(\mathcal{A}, T)$ initialized at $x_0$, satisfies,
    \begin{enumerate}
         \item $\stability(\mathrm{USOLP}(\mathcal{A},T)) = \widetilde\bigO_p( (T^{\gamma}/n)^\frac{1}{\hat{p}_2-1} \cstSmooth R^2 )$,
        \item $f_S(x_T) - f_S(x^*) = \widetilde\bigO_p( \cstSmooth R^2 / T^\gamma )$,
    \end{enumerate}
    for $p \in (1,\infty)$ where $\hat{p}_2 = \max\{p,2\}$ and $T = \sum_{i=1}^r \hat{T}_i$ is the sum of gradient oracle calls from all stages with each stage taking $\hat{T}_i$ gradient oracle calls.
\end{theorem}
The proof, which we provide in \cref{subsec:proof_blackbox_uniform_stabp}, consists of two parts. Firstly, we use a restarting scheme to transfer the algorithm $\mathcal{A}$, which has convergence rate $\hat{\varepsilon}=\bigO(\hat{L} \hat{T}^{-\gamma})$ when applied to the minimization of a convex $(\hat{L}, 2)$-Hölder smooth objective function, to an algorithm with a potentially \emph{faster} convergence rate of $\hat\varepsilon=\bigO(\hat{L} \cstUnifC^{-\frac{2}{p-2}} T^{-\gamma\frac{p-2}{p}})$ when applied to the minimization of a \emph{$(\cstUnifC, p)$-uniformly convex} $(\hat{L},2)$-Hölder smooth objective, where $\hat\varepsilon$ is the wanted accuracy of the algorithm. The second part consists of choosing the regularization parameter $\cstUnifC$ controlling the uniform convexity of the objective on the order of $\cstUnifC = \bigO(\hat\varepsilon)$, which is the prescribed condition on the required accuracy in \cref{lemma:stability_unif_reg}. This choice of $\cstUnifC$ brings the convergence rate of the optimization error to the original rate $\hat{\varepsilon}=\bigO(\hat{L} T^{-\gamma})$ but with the added benefit of being uniformly stable due to the choice of $\cstUnifC$. Similar argument can be made when $p\in(1,2)$ but with the exception of having a $(\hat{L},p)$-Hölder smooth and $(\cstUnifC, p)$-uniformly convex, i.e.\ $\cstUnifC$-strongly convex objective. 

The reduction nearly preserves the optimization convergence rate in the sense that $\bigotildep{p}{\cstSmooth R^2 / T^\gamma}$ is equal up to a constant factor depending on $p$ and a log factor, to the convergence rate $\bigO_p(\cstSmooth R^2 / \hat{T}^\gamma)$ of the optimization algorithm when applied to a smooth convex objective function. Here $T$ denotes the cumulative sum of the number of gradient calls in all $r$ restarting stages and $\hat{T}$ is the number of gradient oracle calls of the optimization algorithm.

\begin{remark}[The case of $p\leq 1+\log^{-1}(d)$]
    The results in \cref{thm:blackbox_uniform_stabp} hold also in the case when $p\leq 1 + \log^{-1}(d)$ up to $\log(d)$ factors. This is a consequence of $\norm{x}_1 = \Theta(\norm{x}_{\hat{p}})$ for $\hat{p} = 1 + \log^{-1}(d)$, so an $L$-smooth function w.r.t. $\norm{\cdot}_1$ is an $\bigO(L)$-smooth function w.r.t. $\norm{\cdot}_{\hat{p}}$. We can then apply the results of \cref{thm:blackbox_uniform_stabp} with $\norm{\cdot}_{\hat{p}}$ and the constants depending on $\hat{p}$ will be modified in the bounds by $\bigO(\log(d))$ factors.
\end{remark}

The rates depending on $\gamma$ will differ based on the optimization algorithm we use. For $p\geq 2$, we have optimization algorithms, such as GeneralizedAGD+ \citep{diakonikolas2021complementary}, whose rates are in the range $\gamma \in (0,1+2/p]$, and employing USOLP reduction on an algorithm with the optimal rate $\gamma = 1+2/p$ yields a learning algorithm that is $\widetilde\bigO_p(T^{1+2/p})$-uniformly stable. When $p\in[1,2)$ the convergence rates are in the range $\gamma \in (0, 3p/2-1]$ where the optimal upper limit translates into in $\widetilde\bigO(T^{1+p^\ast/2})$-uniformly stable algorithms for $1/p^\ast+ 1/p = 1$. The upper limits for both regimes, $p\in[1,2)$ and $p\geq 2$, are attained by  \citep{diakonikolas2021complementary,daspremont2018optimal} and the proof of optimality of the rates can be found in \citep{guzman2015lower}.

The stability rates for the right choice of $p$ and $\cstUnifC$ yield min-max rates that achieve an optimal upper bound for the excess risk up to a constant factor in $p$. In \cref{corr:usolp_excess_risk_low_dim} we show that using USOLP with $p=2$ achieves an excess risk upper bounded on the order of $d^{1/2-1/p}n^{-1/2}$, which is the same as that achieved by the Euclidean black-box reduction algorithm \citep[Theorem 6]{attia2022uniform} when we choose the number of iterations such that the upper bound on the excess risk is minimized.

\begin{corollary}[Excess risk bounds of USOLP for $d\leq n$]\linktoproof{corr:usolp_excess_risk_low_dim}\label{corr:usolp_excess_risk_low_dim}
        Under the assumptions of \cref{thm:blackbox_uniform_stabp}, we have that USOLP applied to the regularized risk minimization in \eqref{eq:optimization_regul} with $\psi(x) = \frac12\norm{x-x_0}_2^2$ and $\cstUnifC$ chosen to be optimal w.r.t.\ smallest excess risk bound as in \cref{thm:risk_upper}, produces an estimate $x_T$ for which the excess risk is upper bounded as
    \begin{equation*}
        \mathbb{E}_S\left[ \delta f(x_T)\right] =\bigO_p\left( \cstSmooth R^2 \frac{d^{1/2-1/\hat{p}}}{n^{1/2}}\right),
    \end{equation*}
    where $\hat{p} = \max\{ 2, p\}$ when the total number of gradient oracle calls is on the order of $T = \widetilde\Omega_p( (d^{\frac12 - \frac{1}{\hat{p}}}n^\frac12 )^\frac1\gamma )$ and $n= \Omega(d^{1/\hat{p}-1/2})$.
\end{corollary}
We provide the proof in \cref{subsec:proof_usolp_excess_risk_low_dim}. The proof consists of applying USOLP with the regularizer $\psi(x) = (1/2)\norm{x-x_0}_2^2$, which by \cref{thm:risk_upper} achieves asymptotic sample complexity on the order of $n^{-1/2}$. The bound between $\ell_p$-norms $\norm{x}_2 \leq d^{1/2-1/\hat{p}} \norm{x}_p$ where $\hat{p} = \max\{p, 2\}$ allows to express the regularity w.r.t.\ $\ell_p$-norm at the cost of an extra $d^{1/2-1/\hat{p}}$ factor. The bound is optimal in the \emph{low-dimensional case when $d\leq n$}, but becomes vacuous in $d$ when the dimension of the problem is large.

The following \cref{corr:usolp_excess_risk_high_dim} shows that in the \emph{high-dimensional case $n > d$}, using USOLP with $p$ based on the $\ell_p$-norm regularity of the function yields an excess risk bound that is \emph{independent of the dimension}, but is on the order of $n^{-1/\hat{p}}$.
\begin{corollary}[Excess risk bounds of USOLP for $d>n$] \linktoproof{corr:usolp_excess_risk_high_dim}\label{corr:usolp_excess_risk_high_dim}
    Under the assumptions of \cref{thm:blackbox_uniform_stabp}, we have that USOLP applied to the regularized risk minimization in \eqref{eq:optimization_regul}, where $\cstUnifC$ is chosen as in \cref{thm:risk_upper}, produces an estimate $x_T$ for which the excess risk is upper bounded as
    \begin{equation*}
        \mathbb{E}_S\left[ \delta f(x_T)\right] =\bigO_p\left( \cstSmooth R^2 \left(\frac{1}{n}\right)^{1/\hat{p}}\right),
    \end{equation*} 
    when the total number of gradient oracle calls is on the order of
    \begin{equation*}
         T = \begin{cases}
             \widetilde\Omega_p( n^{\frac{3-p}{2}\frac1\gamma}) \quad \text{when $p\in (1,2)$}\\
             \widetilde\Omega_p( n^{(1-\frac1p)\frac1\gamma}) \quad \text{when $p\geq 2$}
         \end{cases}
    \end{equation*}
    and $n= \Omega_p(1)$.
\end{corollary}
The proof is in \cref{subsec:proof_usolp_excess_risk_high_dim} and is based on applying \cref{thm:risk_upper} with the optimal choice of $\cstUnifC$ and ensuring that the required number of iterations to reach the necessary accuracy is performed. USOLP achieves optimal excess risk up to a constant factor in $p$ regardless of the convergence rate $\gamma$ of the original algorithm, but for larger $\gamma$ it will require fewer gradient oracle calls to reach it. The condition on $n$ being large enough ensures that the combined smoothness of the regularized objective is on the same order as the smoothness of the unregularized empirical risk.

Our black-box reduction, when applied to an optimization algorithm with a suitably fast convergence rate, achieves the same rates on the excess risk, up to a logarithmic factor, as the optimal non-black-box algorithm \emph{Generalized AGD+} \citep{diakonikolas2021complementary} when applied to an \ERM{} with uniform convex regularization. In the result that follows we show that the convergence rates of Generalized AGD+ in \citep[Theorem 2]{diakonikolas2021complementary} applied to an \ERM{} with uniformly convex regularization combined with the stability results of \cref{lemma:stability_unif_reg} translate to uniform stability rates that are identical up to a logarithmic factor to the ones from our black-box reductions in \cref{thm:blackbox_uniform_stabp} when $p \geq 2$ and we apply the reduction to an optimization algorithm with $\gamma = 1 + 2/p$, e.g.\ the accelerated first order algorithm in \citep{daspremont2018optimal}. This implies that USOLP achieves the best possible rates for the black-box reduction.

\begin{proposition}[Uniform stability of the Generalized AGD+]\linktoproof{thm:diakonikolas_stability_bound}\label{thm:diakonikolas_stability_bound}
    Let $p\geq 2$. Then for a given $T$, the iterate $x_T$ produced by Generalized AGD+ algorithm \citep{diakonikolas2021complementary} applied to the minimization of the regularized empirical risk minimization in \eqref{eq:optimization_regul} with $\cstUnifC = \widetilde\Omega_p(T^{-(1+2/p)} R^{2-p}\cstSmooth)$ is uniformly stable as
    \begin{equation*}
        \stability(\mathcal{A}(T)) = \bigotildepl{p}{ \cstSmooth  R^2 \left(\frac{T^{1+2/p}}{n }\right)^{1/(p-1)}}
    \end{equation*}
    and its optimization error on the empirical risk is upper bounded as
    \begin{equation*}
        f_S(x_T) - f_S(x^*) = \bigotildepl{p}{\cstSmooth R^2    \left(\frac1T\right)^{1+2/p} },
    \end{equation*}
    provided $T =  \Omega_p( R^{2-p} n^{p/(p+2)})$.
\end{proposition}
The proof is given in \cref{subsec:proof_diakonikolas_stability_bound}. For $p=2$ our rates simplify to the stability rates derived by \citep{chen2018stability, attia2022uniform} for the Euclidean case and extend them for $p>2$. The rates match the rates of USOLP in \cref{thm:blackbox_uniform_stabp} when $\gamma = p+2/p$.

\begin{corollary}[Excess risk bound of the Generalized AGD+]\linktoproof{thm:diakonikolas_generalization_bound} \label{thm:diakonikolas_generalization_bound}
    Let $p\geq 2$,  $\ell:\mathbb{R}^d\rightarrow\mathbb{R}$ is convex $\cstSmooth$-smooth function w.r.t.\ $\norm{\cdot}_p$, and $\cstUnifC$ is chosen as in \cref{thm:risk_upper}. Then the Generalized AGD+ algorithm applied to the regularized risk minimization in \eqref{eq:optimization_regul} produces an estimate $x_T$ for which the excess risk is upper bounded as
    \begin{equation*}
        \mathbb{E}_S \left[\delta f(x_T) \right] \leq 16 \cstSmooth R^2 \left(\frac{1}{n}\right)^{1/p},
    \end{equation*}
    when the number of iterations is on the order of $T = \widetilde{\Omega}_p(
    n^{\frac{p-1}{p+2}})= \widetilde{\Omega}_p( 
    n^{(1-\frac1p)\frac{p}{p+2}})$.
\end{corollary}
The proof is given in \cref{subsec:proof_diakonikolas_generalization_bound} and is based on using the optimal choice of $\cstUnifC$ as given in \cref{thm:risk_upper} and ensuring that the Generalized AGD+ algorithm reaches sufficiently low error using its proved convergence rate from \citep[Theorem 2]{diakonikolas2021complementary}. The upper bound matches the result in \cref{corr:usolp_excess_risk_high_dim} for $p\geq 2$ when $\gamma = p + 2/p$.

\section{Classification in $\ell_p$-balls}

We describe an application to binary classification where the regularity of the data and the optimal estimator are given in non-Euclidean geometry. In this setup the generalization guarantees in \cref{corr:usolp_excess_risk_high_dim} improve upon the ones given by posing Euclidean geometry regularity on the problem.

Let $p\geq2$ and $p^\ast = p/(p-1)$. The classification task consists of $n$ labeled points $(a_i, b_i) \in \R^d \times \{-1, 1\}$, $i \in [n]$ where $a_i \in \ballpzero[R][p^\ast]$. Consider the generalized linear model $h(x) \defi \frac{1}{n}\sum_{i\in[n]} h_i(x) \defi \frac{1}{n}\ \sum_{i\in[n]} f(b_i\innp{a_i, x})$ where $f:\R \to \R$ is convex, $L$-smooth, and twice differentiable, e.g.,\ the logistic loss, the squared loss $f(t) = (1-t)^2$, or the smoothed hinge loss \citep{rennie2005loss}. Since $\nabla h_i(x) = f'(b_i\innp{a_i, x})b_ia_i$ and the bound on Lipschitzness in \eqref{eq:smooth_to_lipschitz}, we have
\begin{align*}
    \begin{aligned}
        \innp{\nabla^2 h(x)& v, v}  = \frac{1}{n} \sum_{i=1}^n\abs{f''(b_i\innp{a_i, x})} \bigg(\sum_{j=1}^d v_j a_{i,j}\bigg)^2  \\
        &\leq L \norm{v}_{p}^2 \sum_{i=1}^n\frac{\norm{a_i}_{p^\ast}^2}{n} \leq LR^2 \norm{v}_p^2,
    \end{aligned}
\end{align*}
where we used the Hölder's inequality and the assumption on the data. As a result $h$ is $LR^2$-smooth with respect to the $\norm{\cdot}_p$ norm.

Since for $p^\ast\leq 2$, we have $\norm{x}_2 \leq \norm{x}_{p^\ast}$, the data $a_i$ is also in a Euclidean ball of radius $R$ and, as a result, the function is also $(LR^2)$-smoothness w.r.t.\ $\norm{\cdot}_2$. The excess risk bound in the Euclidean case is $\bigo{n^{-1/2}}$, which is a factor of $n^{1/2-1/p}$ lower than the upper bound of $\bigo{n^{1/p}}$ derived in \cref{thm:risk_upper} using non-Euclidean regularization. However, depending on the problem setup, the distance to a minimizer measured with the Euclidean distance can be a factor $d^{\frac{1}{2}-\frac{1}{p}}$ greater than when measured in $\ell_p$-norm, with the upper bound reached when the data is proportional to the vector of ones. Thus, in the high-dimensional case $d \gg n$, using the $p$-norm can lead to statistical bounds in \cref{thm:risk_upper} that are up to $(d/n)^{1/2-1/p}$ lower compared to using the Euclidean regularity.

In the case when $p\in(1,2)$, the $\ell_p$-norm regularized problem can have much higher smoothness that can have beneficial properties for the optimization algorithms while keeping the same dependence on $n$. However, in this case, the distance to the minimizer measured in $\ell_p$-norm can be a factor of the dimension greater than when measured in the Euclidean distance. Same as when $p\geq 2$, the improvement will depend on where the data and minimizer are located in terms of the $\ell_p$-ball.

\section{Conclusion}

In this paper we studied uniform stability properties of first-order algorithms in non-Euclidean settings, i.e., for functions whose regularity is measured w.r.t.\ $\norm{\cdot}_p$ norms. We developed new bounds for the uniform stability of approximate minimizers of empirical risk with uniformly convex regularization. The stability bounds, in combination with an observation that a restart scheme for minimization of uniformly convex objectives results into convergence rates that capture the uniform convexity of the objective, allowed us to design a black-box reduction scheme that gives an optimal trade-off between stability and convergence rate. We also showed that the black-box reduction yields optimal expected excess risk up to a constant factor in $p$. We provided a derivation of an excess risk lower bound for $p\geq 2$ in the high-dimensional setting.

\acks{Simon Vary and Patrick Rebeschini were funded by UK Research and Innovation (UKRI) under the UK government’s Horizon Europe funding guarantee [grant number EP/Y028333/1].}

\printbibliography[heading=bibintoc] 

\newpage
\appendix
\section{\texorpdfstring{{Supporting lemmata}}{Supporting lemmata}}

\begin{lemma}[Restricted smoothness and strong convexity of $\frac1p\norm{x}_p^p$ in the $p$-ball] \label{lemma:p_norm_ball_smooth_str_convex}
    Let $\psi(x) = \frac1p\norm{x}_p^p$ defined for $x\in\ballpzero[R]$. Then $\psi(x)$ is
    \begin{enumerate}
        \item{$(p-1)R^{p-2}$-smooth w.r.t\ $\ell_p$-norm inside of $\ballpzero[ R]$ when $p\geq 2$, and }
        \item{ $(p-1)R^{p-2}$-strongly convex w.r.t\ $\ell_p$-norm inside of $\ballpzero[R]$ when $p\in(1,2)$.}
    \end{enumerate} 
\end{lemma}
\begin{proof}
    We have $\nabla^{2}\psi(x) = (p-1)\operatorname{diag}((|x_i|^{p-2})_{i=1}^d)$. When $p\geq 2$, we can bound 
    \begin{align*}
        \begin{aligned}
            \innp{\nabla^2 \psi(x)v, v} = (p-1)\sum_{i=1}^d v_i^2 |x_i|^{p-2} \leq (p-1) \norm{v}_p^2 \norm{x}_{p}^{p-2} \leq (p-1) R^{p-2} \norm{v}_p^2
        \end{aligned},
    \end{align*}
    where we used Hölder's inequality $\innp{u, w} \leq \norm{u}_q\norm{w}_{q^{\ast}}$, with $q = p/2$ and so $q^\ast = p/(p-2)$, and $\norm{x}_p = R$ by $x\in\ballpzero[R]$. Consequently, by Taylor approximation, we have for all $x\in\conv{y, z}$, where $y,z\in\ballpzero[0, R]$, that
    \begin{equation*}
        \psi(y) \leq \psi(z) + \innp{\nabla\psi(z), y-z} + \frac{1}{2}\innp{\nabla^2 \psi(x)(y-z), y-z} \leq \psi(z) + \innp{\nabla\psi(z), y-z} + \frac{(p-1)R^{p-2}}{2}\norm{y-z}_p^2.
    \end{equation*}
    When $p \in (1,2)$, we can bound
    \begin{align*}
        \begin{aligned}
            \innp{\nabla^2 \psi(x)v, v} &= (p-1)\sum_{i=1}^d v_i^2 |x_i|^{2-p}\\
            &\geq (p-1) \left(\sum_{i=1}^d \abs{v_i}^{2/q} \right)^q \left(\sum_{i=1}^d |x_i|^\frac{2-p}{1-q}\right)^{1-q}\\
            &= (p-1) \left(\sum_{i=1}^d |v_i|^p \right)^{2/p} \left(\sum_{i=1}^d |x_i|^p\right)^\frac{p-2}{p} \\
            &= (p-1) \norm{v}^2_p \norm{x}^{p-2}_p \geq (p-1) \norm{v}^2_p \left(\frac{1}{R}\right)^{2-p},
        \end{aligned}
    \end{align*}
    where the first line comes from the form of $\nabla^2 \psi(x)$, the second line uses the reverse Hölder inequality for $q\geq 1$, in the third line we choose $q = 2/p\geq 1$, and the final line comes from $x\in\ballpzero[R]$ and $p \in (1,2)$.
\end{proof}

\begin{lemma} \label{lemma:unif_cvxty_of_norm_p_to_the_p}
    For $p \geq 2$ and $x_0\in\R^d$, the regularizer $\psi(x) \defi \frac{2^{p-2}}{p}\norm{x-x_0}_p^p$ is $(1, p)$-uniformly convex with respect to $\norm{\cdot}_p$ in $\R^d$. For $p\in(1,2)$, the regularizer $\psi(x) = (1/p)\norm{x-x_0}_p^p$ is $(2^{3-2p}(p/(p-1))^{p-1},p)$-Hölder smooth.
\end{lemma}

The fact that $\psi(x)$ is uniformly convex is classical, see for instance \citep{zalinescu1983uniformly}, that established $(2^{-\frac{p(p-2)}{p-1}},p)$-uniform convexity of $\frac{1}{p}\norm{x}_p^p$. In the following proof, we show the uniform convexity of $\psi$ with a better constant. 
\begin{proof}
    For $p\geq 2$. Note that $\norm{x}_p^p$ is a separable function, i.e., it has the form $\sum_{i=1}^d f_i(x_i)$. Thus, it is enough to show the uniform convexity of the one-dimensional case and add up all of the corresponding inequalities in order to obtain the result. In \citep[Lemma 4.2.3]{nesterov2018lectures}, it is established that $\frac{2^{p-2}}{p}\norm{x}_2^p$ is $(1, p)$-uniformly convex with respect to the Euclidean norm $\norm{\cdot}_2$. Since in one dimension, all of the $p$-norms are the same, the $1$-D result is proven, and the result follows.

    For $p\in(1,2)$. Let $q = p/(p-1)$, we have that $q> 2$ and $(1/q)\norm{x}_q^q$ is $(2^{2-q}, q)$-uniformly convex by the first part of the proof. By  similar derivations as in \citep[Example 3.27]{Boyd2004Convex} but with Hölder inequality, we get that the Fenchel dual of $(1/q) \norm{x}_q^q$ is $(1/p) \norm{x}_p^p$. From \cref{lemma:duality_holder_unifconvex} combined with $(1/q) \norm{x}_q^q$ being $(1,q)$-uniformly convex, we have that $(1/p)\norm{x}_p^p$ is $(2^{3-2p}(p/(p-1)^{p-1}),p)$-Hölder smooth.
\end{proof}

\begin{lemma}[Distance bound]\label{lemma:distance_bound}
     Let $\psi(x)$ be an $(\cstUnifC,p)$-uniformly convex regularizer of the empirical risk $f_S(x)$, i.e., $f_S^{(\cstUnifC)}(x) = f_S(x) + \psi(x)$, and define $x^\ast = \argmin_{x\in\mathcal{X}} f_S(x)$, $x_\cstUnifC^\ast = \argmin_{x\in\mathcal{X}} f_S^{(\cstUnifC)}(x)$, and $x_0 = \argmin_{x\in\mathcal{X}}\psi(x)$. Then, if $\psi(x^*) - \psi(x_0)\leq D^p$ for $D > 0$, we also have that
    \begin{equation*}
         \mathrm{D}_\psi (x_\cstUnifC^*, x_0) \leq D^p,
    \end{equation*}
    where $\mathrm{D}_\psi$ is the Bregmann divergence of $\psi$
\end{lemma}
\begin{proof}
If we assume that we have $\psi(x^*) - \psi(x_0)\leq D^p$ for $D>0$, we get 
\begin{align} \label{eq:regul_bound}
    \begin{aligned}
    \mathrm{D}_\psi (x_\cstUnifC^*, x_0) &= \psi(x_\cstUnifC^*) - \psi(x_0) - \langle \nabla \psi(x_0) , x_0 - x_\cstUnifC^*\rangle \\ 
    &\leq \psi(x_\cstUnifC^*) - \psi(x_0)\\
    &\leq \psi(x_\cstUnifC^*) - \psi(x_0) + \frac1\cstUnifC \left( f_S(x_\cstUnifC^*) - f_S(x^*)\right)\\
    &\leq \psi(x_\cstUnifC^*) - \psi(x_0) + \psi(x^*) - \psi(x_\cstUnifC^*) = \psi(x^*) - \psi(x_0) \leq D^p, 
    \end{aligned}
\end{align}
where the first inequality comes from $x_0$ being the minimum of $\psi$, the second inequality from $x^*$ being minimum of $f_S$, the third inequality from $x_\cstUnifC^*$ being the minimum of $f^{(\cstUnifC)}_S$, and the last fourth inequality comes from the definition of $D$.
\end{proof}

The following is a generalization of \citep[Lemma 1]{attia2022uniform} for uniformly convex functions.
\begin{lemma}[Upper bound on distance between perturbed minima]\label{lemma:minima_distance_bound}
Let $f_1$ be convex and  $f_2$ be $(\cstUnifC,p)$-uniformly convex w.r.t.\ the norm $\| \cdot \|$. For $x_1\in\argmin_x f_1(x)$ and $x_2\in\argmin_x f_2(x)$ we have
\begin{equation*}
    \|x_1 - x_2 \| \leq  \left( \frac{p}{\cstUnifC} \| \nabla f_1(x_1) - \nabla f_2(x_1) \|_* \right)^{\frac{1}{p-1}},
\end{equation*}
where $\| \cdot\|_*$ is the dual norm to $\| \cdot\|$.
\end{lemma}
\begin{proof}
    By $f_2$ being $(\cstUnifC,p)$-uniformly convex and having minimum in $x_2$ we have
    \begin{align*}
        \nabla f_2(x_1)^\top \left(x_1 - x_2\right) \geq f_2(x_1) - f_2(x_2) + \frac{\cstUnifC}{p} \| x_2 - x_1\|^p \geq \frac{\cstUnifC}{p} \| x_2 - x_1\|^p.
    \end{align*}
    
    The first-order optimality of $x_1$ for $f_1$ implies that $\nabla f_1(x_1)^\top (x_1 - x_2) \leq 0$, thus
    \begin{align*}
        \nabla f_2(x_1)^\top (x_1 - x_2) &= \nabla f_1(x_1)^\top (x_1 - x_2) + \left(\nabla f_2(x_1) - \nabla f_1(x_1)\right)^\top (x_1 - x_2)\\
        &\leq \left(\nabla f_2(x_1) - \nabla f_1(x_1)\right)^\top (x_1 - x_2).
    \end{align*}
    Applying the H\"older inequality to the above and combining together yields the result
    \begin{equation*}
        \| x_2 - x_1\| \leq \left(  \frac{p}{\cstUnifC}  \| \nabla f_2(x_1) - \nabla f_1(x_1) \|_* \right)^{\frac{1}{p-1}}.
    \end{equation*}
\end{proof}

\subsection{Proof of \cref{lemma:duality_holder_unifconvex} \label{subsec:proof_duality_holder_unifconvex}}
\begin{proof}\linkofproof{lemma:duality_holder_unifconvex}
    Let $\Pi(y) = \argmax_{x\in \mathcal{X}} \innp{x,y}-f(x)$. Take $y_1, y_2$ and $\gamma\in(0,1)$ and denote $x_1 = \Pi(y_1)$, $x_2 = \Pi(y_2)$, and $x_\gamma = \gamma x_1 + (1-\gamma)x_2$. From \citep[Lemma 15 (b)]{shalev2007online}, we have $\gamma_1\in\partial f(x_1)$ and $\gamma_2\in\partial f(x_2)$. From $f$ being $(\cstUnifC, p)$-uniformly convex we get
    \begin{align*}
        f(x_\gamma) - f(x_1) - \innp{y_1, x_\gamma - x_1} &\geq \frac{\cstUnifC}{p} \| x_\gamma - x_1 \|^p \\
        f(x_\gamma) - f(x_2) - \innp{y_2, x_\gamma - x_2} &\geq \frac{\cstUnifC}{p} \| x_\gamma - x_2 \|^p.
    \end{align*}

    Adding $\gamma$ times the first inequality to $(1-\gamma)$ times the second inequality yields
    \begin{equation}
        f(x_\gamma) - \left(\gamma f(x_1) + (1-\gamma)f(x_2)\right) + \gamma (1-\gamma)\innp{y_2-y_1, x_2 - x_1} \geq \gamma (1-\gamma) \frac{\cstUnifC}{p} \| x_1 - x_2 \|^p. \label{eq:lemma_holder_ineq1a}
    \end{equation}
    However, by $f$ being $(\cstUnifC,p)$ uniformly convex we also have
    \begin{equation}
        f(x_\gamma) - \left(\gamma f(x_1) + (1-\gamma)f(x_2)\right) \leq -\gamma (1-\gamma) \frac{\cstUnifC}{p} \| x_1 - x_2 \|^p. \label{eq:lemma_holder_ineq1b}
    \end{equation}
    
    Subtracting \eqref{eq:lemma_holder_ineq1b} from \eqref{eq:lemma_holder_ineq1a} yields
    \begin{equation*}
        \innp{y_2 - y_1, x_2 - x_1} \geq \frac{2\cstUnifC}{p} \| x_2 - x_1\|^p,
    \end{equation*}
    which in combination with the H\"{o}lder inequality $\innp{y_2 - y_1,x_2-x_1}\leq \|x_2 - x_1 \| \| y_2 - y_1\|_*$ gives that
    \begin{equation*}
        \frac{2\cstUnifC}{p} \| x_1 - x_2\|^{p-1} \leq \| y_1 - y_2 \|_*
    \end{equation*}
    or
    \begin{equation}
        \| \nabla f^*(y_1) - \nabla f^*(y_2)\|\leq  \left(\frac{p}{2\cstUnifC}\right)^{1/(p-1)} \| y_1 - y_2 \|_*^{1/(p-1)}. \label{eq:bound_norm_grad_dual}
    \end{equation}

    Let $T\in\mathbb{N}$. For all $t\in\{0,1,\ldots, N\}$ define $\beta_t = t/T$. We have
    \begin{align*}
        f^*(y+\lambda) - f^*(y) &= f^*(y+\beta_T \lambda) - f^*(y + \beta_0\lambda) \\
        &= \sum_{t=0}^{T-1} f^*(y+\beta_{t+1}\lambda) - f^*(y+\beta_T\lambda).
    \end{align*}

    From convexity of $f^*$ we have
    \begin{equation*}
        f^*(y + \beta_{t+1}) - f^*(y + \beta_t) \leq \innp{\nabla f^*(y+ \beta_{t+1}\lambda), (\beta_{t+1} - \beta_t)\lambda} = \frac1T \innp{\nabla f^*(y+\beta_{t+1}\lambda), \lambda}
    \end{equation*}

    Thus we have
    \begin{align*}
        \innp{\nabla f^*(y+\beta_{t+1}\lambda), \lambda} &= \innp{\nabla f^*(y), \lambda} + \innp{\nabla f^*(y + \beta_{t+1}) - \nabla f^*(y), \lambda} \\
        &\leq \innp{\nabla f^*(y), \lambda} + \| \nabla f^*(y+\beta_{t+1}\lambda) - \nabla f^*(y)\| \| \lambda \|_* \\
        &\leq \innp{\nabla f^*(y), \lambda} + \left( \frac{p}{2\cstUnifC} \| \beta_{t+1}\lambda \|_* \right)^{1/(p-1)} \| \lambda \|_* \\
        &= \innp{\nabla f^*(y), \lambda} + \left(\frac{p}{2\cstUnifC} \beta_{t+1}\right)^{1/(p-1)} \| \lambda\|_*^{p/(p-1)},
    \end{align*}
    where the second inequality is the consequence of Hölders inequality and the third inequality comes from \eqref{eq:bound_norm_grad_dual}.

    Now, we can express
    \begin{align*}
        f^*(y+\lambda) - f^*(y) \leq \innp{\nabla f^*(y), \lambda} + \left(\frac{p}{2\cstUnifC}\right)^{1/(p-1)} \| \lambda\|_*^{p/(p-1)}  \frac{1}{T} \sum_{t=0}^{T-1} (\beta_{t+1})^{1/{p-1}}.
    \end{align*}

    We are interested in the asymptotic bound as $T\to\infty$, for which we have that the limit is an Riemannian sum that can be expressed as an integral
    \begin{equation*}
        \lim_{T\to\infty} \frac{1}{T} \sum_{t=0}^{T-1} (\beta_{t+1})^{1/{p-1}} = \lim_{T\to\infty} \frac{1}{T} \sum_{t=0}^{T-1} \left(\frac{t}{T}\right)^{1/{p-1}} = \int_0^1 x^\frac1{p-1} \mathrm{d}x = \frac{p-1}{p}.
    \end{equation*}

    Thus we have 
     \begin{align*}
        f^*(y+\lambda) - f^*(y) \leq \innp{\nabla f^*(y), \lambda} + \frac{p-1}{p}\left(\frac{p}{2\cstUnifC}\right)^{1/(p-1)} \| \lambda\|_*^{p/(p-1)},
    \end{align*}
    which after substituting $p = \frac{q}{q-1}$ and $\lambda = x-y$ yields the result.
\end{proof}

\section{\texorpdfstring{{Proofs for results in Section \ref{sec:unif_convex_generalization}}}{Proofs for results in Section \ref{sec:unif_convex_generalization}}}

\subsection{Proof of \cref{lemma:stability_unif_reg}\label{subsec:proof_stability_unif_reg}}
\begin{proof}\linkofproof{lemma:stability_unif_reg}
Let $\hat{x}$ and $\tilde x$ be two $\hat\varepsilon$-minimizers of $f^{(\cstUnifC)}_S$ and $f^{(\cstUnifC)}_{S'}$ respectively that are $(\cstUnifC,p)$-uniformly convex w.r.t.\, $\norm{\cdot}$, where $S$ and $S'$ differ only in at most one sample. Let the minimum of $f^{(\cstUnifC)}_S$ be attained at $\hat{x}^*_\cstUnifC$ and the minimum of $f^{(\cstUnifC)}_{S'}$ is attained  at $\tilde x^*_\cstUnifC$. By the triangle inequality we have
\begin{equation}
    \| \hat{x} - \tilde x \|\leq \| \hat{x}^*_\cstUnifC - \tilde x^*_\cstUnifC \| + \| \hat{x} - \hat{x}^*_\cstUnifC \| + \| \tilde x - \tilde x^*_\cstUnifC \|.\label{eq:distance_iterates} 
\end{equation}

We denote the first term  of the right hand side in \eqref{eq:distance_iterates} as $S_1$ and upper bound it
\begin{equation}
    S_1 = \| \hat{x}^*_\cstUnifC - \tilde x^*_\cstUnifC \| \leq \left( \frac{\nu}{n \cstUnifC} \| \nabla \ell(\hat{x}_\cstUnifC^*; z_i) - \nabla \ell(x_\cstUnifC^*; z_i') \|_* \right)^{1/(\nu-1)} \leq \left( \frac{2\nu}{n\cstUnifC} \cstLip \right)^{1/(\nu-1)}, \label{eq:error_minima_distance}
\end{equation}
where the first inequality follows from \cref{lemma:minima_distance_bound} applied to the regularized \ERM{} functions that are uniformly convex w.r.t.\ $\norm{\cdot}$ and the second inequality comes from the triangle inequality combined with the Lipschitz constant $\cstLip$ for $\ell(\cdot; z)$ bounding the dual norm of $\nabla \ell ( \hat{x}_\cstUnifC^*; z_i)$ and $\nabla \ell (\tilde x_\cstUnifC^*; z'_i)$.

We denote the second term of the right hand side in \eqref{eq:distance_iterates} as $S_2$ and upper bound as follows 
\begin{equation} \label{eq:error_convergence}
    S_2 = \norm{\hat{x} - \hat{x}_\cstUnifC^\ast} \leq \left( \frac{\nu}{\cstUnifC}\left( f_S^{(\cstUnifC)}(\hat{x}) - f_S^{(\cstUnifC)}(\hat{x}_\cstUnifC^\ast) \right) \right)^{1/\nu}\leq \left( \frac{\nu}{\cstUnifC}\hat\varepsilon  \right)^{1/\nu},
\end{equation}
where the first inequality comes from the definition of $f_S^{(\cstUnifC)}(x)$ and the fact that $x_\cstUnifC^\ast$ is its minimizer, and the second inequality is the consequence of $\hat{x}$ being $\hat\varepsilon$-minimizer of $f_S^{(\cstUnifC)}$.  Note that the third term of the right hand side in \eqref{eq:distance_iterates}, which we denote as $S_3$, is bounded analogously as $S_2$ in \eqref{eq:error_convergence}.

When $\hat\varepsilon \leq (\nu/\cstUnifC)^{1/(\nu-1)} (2\cstLip/n)^{\nu/(\nu-1)}$, we have that the upper bound on $S_2$ in \eqref{eq:error_convergence} is smaller than the upper bound on $S_1$ in \eqref{eq:error_minima_distance}. In this case the stability is bounded for all $z\in\mathcal{Z}$ as
\begin{equation*}
    \abs{\ell(\hat{x}; z) - \ell(\tilde{x}; z)} \leq \cstLip \| \hat{x} - \tilde x\|\leq  3\cstLip  \left( \frac{2\nu}{n\cstUnifC} \cstLip \right)^{1/(\nu-1)}.
\end{equation*}

When the error of the regularized problem is small $\hat\varepsilon \leq \cstUnifC D^\nu$, we can upper bound the optimization error of the non-regularized problem as
\begin{align*}
 \begin{aligned}
    f_S(\hat{x}) - f_S(x^*) &= f^{(\cstUnifC)}_S(\hat{x}) - f^{(\cstUnifC)}_S(x^*) + \cstUnifC\psi(x^*) - \cstUnifC\psi(\hat{x})\\
     &\circled{1}[\leq] f^{(\cstUnifC)}_S(\hat{x}) - f^{(\cstUnifC)}_S(x_\cstUnifC^*) + \cstUnifC\psi(x^*) - \cstUnifC\psi(x_0)\\
    & \circled{2}[\leq]  \hat\varepsilon + \cstUnifC D^\nu \\
    &\leq 2 \cstUnifC D^\nu,
 \end{aligned}
\end{align*}
where $\circled{1}$ holds by $x_\cstUnifC^\ast$ being the minimizer of $f_S^{(\cstUnifC)}$ while $x_0$ is the minimizer of $\psi(x)$ and $\circled{2}$ comes from \cref{lemma:distance_bound} and the definition of $D$.
\end{proof}

\subsection{Proof of \cref{thm:risk_upper}\label{subsec:proof_risk_upper}}
\begin{proof}\linkofproof{thm:risk_upper}
    Let $\mathcal{A}$ be an algorithm that for a dataset $S$ returns an $\hat\varepsilon$-minimizer $\hat{x}$  of the regularized \ERM{} $f_S^{(\cstUnifC)}$ with $\cstLip$-Lipschitz losses $\ell(\cdot,z)$.
    By the excess risk decomposition in \eqref{eq:excess_risk_decomp1} and \cref{lemma:stability_unif_reg} we have
    \begin{equation*}
        \mathbb{E}_S \left[ \delta f(\hat x)\right] \leq \mathcal{E}_{\mathrm{stab}}(\mathcal{A}, S) + \mathbb{E}_S\left[ f_S(\hat{x}) - f_S(x^\ast) \right] \leq 3 \left( \frac{2\nu}{n\cstUnifC} \cstLip^\nu \right)^{1/(\nu-1)} + 2 \cstUnifC D^\nu,
    \end{equation*}
    where $x^\ast = \argmin_{x\in\R^d} f_S(x)$.
    
    To find the value of $\cstUnifC\geq 0$ that minimizes the upper bound we set its derivative in $\cstUnifC$ to zero, since the expression is convex on $\cstUnifC \geq 0$:
    \begin{equation*}
        0= 2 D^\nu - 3\left(\frac{2\nu}{n}\right)^\frac{1}{\nu-1} \cstLip^\frac{\nu}{\nu-1} \frac{1}{\nu-1}\cstUnifC^{-\frac{\nu}{\nu-1}},
    \end{equation*}
    which after rearranging yields the minimum is attained for
    \begin{equation}\label{eq:alpha_minimum}
        \cstUnifC = \left(\frac{3}{\nu-1}\right)^{1-\frac1\nu} 2^{\frac{2}{\nu} - 1} \nu^\frac{1}{\nu-1} \left(\frac{1}{n}\right)^\frac{1}{\nu} D^{1-\nu} \cstLip  > 0
    \end{equation}
    and the value of the excess risk is upper bounded as
    \begin{align*}
        \mathbb{E} \left[ \delta f(\hat x)\right] &\leq\mathcal{E}_{\mathrm{stab}}(\mathcal{A}) + \mathcal{E}_{\mathrm{opt}} \leq 2^{1+\frac{2}{\nu}}\left(\frac{3}{\nu-1}\right)^{1-\frac1\nu} \left(\frac{1}{n}\right)^\frac1\nu D \cstLip \\
        &\leq 8 \left(\frac{1}{n}\right)^\frac1\nu D \cstLip,
    \end{align*}
    where in the first line we substituted $\cstUnifC$ from \eqref{eq:alpha_minimum}, and the second line follows from $2^{1+2/\nu}(\frac{3}{\nu-1})^{1-\frac1\nu}\leq 8$ for $\nu\geq 2$.

    It remains to ensure the conditions of \cref{lemma:stability_unif_reg} are met, i.e.\ the error $\hat\varepsilon$ of the approximate solution $\hat{x}$, is small enough, and in particular bounded by the minimum of two terms. The bound on the first of the two terms in the minimum is met when 
    \begin{align*}
        \hat\varepsilon \leq \cstUnifC D^\nu &= 2^{2/\nu-1}\left(\frac{3}{\nu-1}\right)^{1-\frac1\nu} \left(\frac{1}{n}\right)^\frac{1}{\nu} D \cstLip\\
        &\leq 2 \left(\frac{1}{n}\right)^\frac{1}{\nu} D \cstLip.
    \end{align*}
    where in the first line we substituted $\cstUnifC$ from \eqref{eq:alpha_minimum}, and the second line follows from $(\frac{3}{\nu-1})^{1-1/\nu}\leq 2$ and $2^{2/\nu - 1} \leq 1$ for $\nu\geq 2$. The second error condition of \cref{lemma:stability_unif_reg} is 
    \begin{align*}
        \hat\varepsilon \leq \left(\frac{\nu}{\cstUnifC}\right)^{1/(\nu-1)} \left(\frac{2\cstLip}{n}\right)^{\nu/(\nu-1)} & = 2^{2/\nu} \left(\frac{\nu-1}{3}\right)^{\frac1\nu} \cstLip D \left(\frac{1}{n}\right)^{1+\frac1\nu} \\
        &\leq 6\cstLip D\left(\frac{1}{n}\right)^{1+\frac1\nu},
    \end{align*}
    where the second line follows from $2^{2/\nu}(\frac{\nu-1}{3})^{1/\nu}\leq 6$ for $\nu\geq 1$. 
\end{proof}

\subsection{Proof of \cref{thm:risk_lower}\label{subsec:proof_risk_lower}}
\begin{proof}\linkofproof{thm:risk_lower}
     We consider two cases, when $1 \leq p \leq 1+1/\log(d)$ and when $p>1+1/\log(d)$.

    \textbf{Case 1 ($1 \leq p \leq 1+1/\log(d)$):}\\
    For $v\in\{\pm e_i\}_{i=1}^s$, where $e_i$'s denote the first $s$ canonical basis vectors, we define the distribution $Z\sim P_v$ as
    \begin{equation*}
        Z_j = \begin{cases}
            1, \qquad &\text{with probability}\ \frac{1+\delta v_j}{2}\\
            -1, \qquad &\text{with probability}\ \frac{1-\delta v_j}{2} 
        \end{cases}\quad \text{for}\,j\in\{1,\ldots, s\},
    \end{equation*}
    and $Z_j = 0$ for $j\in \{s+1, \ldots, d\}$. The distribution is supported only on the first $s\leq d$ entries and is well defined when $\delta \leq 1$. We have that $\mathbb{E}_{z\sim P_v} = \delta v$. 
    
    Let $\ell(x;z) = \cstLip s^{-1/p^\ast} x^\top z$ be a linear loss where $1/p+1/p^\ast=1$. The population loss for $v\sim P_v$ is
    \begin{equation*}
        \ell_v(x) = \mathbb{E}_{z\sim P_v} \ell(x;z) = \delta s^{-1/p^\ast} x^\top z
    \end{equation*}
    and its infimum for $x\in\ballpzero(r)$ is
    \begin{equation*}
        \ell_v^\ast := \inf_{x\in\ballpzero(r)} \ell_v(x) = -\delta s^{-1/p^\ast}.
    \end{equation*} 

    For $v\neq v'$ we have
    \begin{align*}
        \mathrm{d_{\mathrm{opt}}} (v, v', \ballpzero[r]) = \inf_{x\in\ballpzero[r]} \ell_v(x) + \ell_{v'}(x) - \ell_v^\ast - \ell_{v'}^\ast &= s^{-1/p} \delta \inf\left( (v + v')^\top x + 2\right) \\
        &=\delta s^{-1/p^\ast} (2 - \norm{v+v'}_{p^\ast})\\
        &=(2-\sqrt{2})\delta s^{-1/p^\ast}.
    \end{align*}

    From \citep[Lemma 2, Sec.~A.1]{levy2019necessary} combined with Fano's inequality given in \citep[Proposition 7]{levy2019necessary} yields a lower bound on the excess risk as
    \begin{align}
        \inf_{\hat{x}\in\ballpzero[r]} \sup_{P\in\mathcal{P}}\mathbb{E}_{S\sim P^n}  \left[ \delta f(\hat x)\right] &\geq \frac{r}{4} \cstLip \delta s^{-1/p^\ast}\left(1-\frac{3n\delta^2 + \log 2}{\log(2s)}\right). \label{eq:lowerbound_delta_psmall}
    \end{align}
        
    Set $\delta = \sqrt{\log(s)/(6 n)}$, which yields
     \begin{align*}
        \inf_{\hat{x}\in\ballpzero[r]} \sup_{P\in\mathcal{P}}\mathbb{E}_{S\sim P^n}  \left[ \delta f(\hat x)\right] &\geq \frac{r}{4} \cstLip  s^{-1/p^\ast}  \sqrt{\frac{\log(s)}{n}}\left(1- \frac{\frac{32}{6} \log(s) + \frac{4}{3}\log 2}{s}\right)\\
        &\geq \frac{r}{4} \cstLip  s^{-1/p^\ast}  \sqrt{\frac{\log(s)}{n}},
    \end{align*}
    where the second inequality holds for $s\geq 8$ which ensures that the right hand side remains positive. Note that for $p \leq 1 + 1/\log(d)$ and $s\leq d$, we have that $p \leq 1 + 1/\log(s)$, for which we can bound $s^{-1/p^\ast}  = s^{-(p-1)/p} \geq 1/e$ completing the proof
    \begin{equation*}
        \inf_{\hat{x}\in\ballpzero[r]} \sup_{P\in\mathcal{P}}\mathbb{E}_{S\sim P^n}  \left[ \delta f(\hat x)\right] \geq \frac{r}{4e} \cstLip   \sqrt{\frac{\log(s)}{n}} \geq \frac{r \cstLip}{12} \sqrt{\frac{\log(s)}{n}}
    \end{equation*}
    when we choose $s =d$. In order for $\delta \leq 1$, we need that $\log(d) \leq 6n$. When $\log(d) \geq 6n$, we choose $\delta = 1$ and get
    \begin{equation*}
        \inf_{\hat{x}\in\ballpzero[r]} \sup_{P\in\mathcal{P}}\mathbb{E}_{S\sim P^n}  \left[ \delta f(\hat x)\right] \geq \frac{r \cstLip}{12}.
    \end{equation*}
    
    \textbf{Case 2 ($p > 1+1/\log(d)$):}\\
    Let $\ell(x;z) = \cstLip x^\top z$ be a linear loss. For $v\in\{\pm 1\}^d$ we define the distribution $Z\sim P_v$ as
    \begin{equation*}
        Z = \begin{cases}
            v_j e_j, \qquad &\text{with probability}\ \frac{1+\delta}{2s}\\
            -v_j e_j, \qquad &\text{with probability}\ \frac{1-\delta}{2s}
        \end{cases}\quad \text{for}\,j\in\{1,\ldots, s\},
    \end{equation*}
    supported only on the first $s\leq d$ entries. The distribution is well defined when $\delta \leq 1$.
    
    We have that
    \begin{align*}
        \ell_v(x) = \mathbb{E}_{z\sim P_v}\left[\ell(x;z) \right] &= \sum_{j=1}^s \frac{1+\delta}{2s}\ell(x, v_j e_j) + \frac{1-\delta}{2s}\ell(x, -v_j e_j) \\
        &= \cstLip\sum_{j=1}^s \frac{1}{2s} x_j v_j \left(1+\delta - 1 + \delta\right) \\
        &= \frac{\cstLip \delta}{s} x^\top v_{\Omega},
    \end{align*}
    where we denote $v_\Omega$ to be a vector that contains entries of $v$ at indices $\Omega= \{1,\ldots s\}$ and zeroes otherwise. By duality the minimum of this function can be computed as
    \begin{equation*}
        \ell_v^* =  \min_{x\in \ballpzero[r]} \frac{\cstLip\delta}{d} x^\top v_\Omega = \min_{x\in \ballpzero[1]} \frac{\cstLip\delta r}{d} x^\top v_\Omega = - \frac{\cstLip\delta r}{s} \| v_\Omega \|_{p^\ast}, 
    \end{equation*}
    where $1/p^\ast + 1/p = 1$.

    For $v, v' \in \left\{ \pm 1\right\}^d$, we have
    \begin{align*}
        \mathrm{d}_\mathrm{opt}(v,v',\ballpzero[r] ) :&= \inf_{x\in \ballpzero[r] } \ell_v(x) - \ell_v^* + \ell_{v'}(x) - \ell_{v'}^* \\
        &=  \inf_{x\in \ballpzero[1]}  \cstLip\frac{\delta r}{s}\left( x^\top(v_\Omega+v'_\Omega) + \| v_\Omega \|_{p^\ast} + \| v'_\Omega \|_{p^\ast}\right)\\
        & = \cstLip\frac{\delta r}{s}\left( \| v_\Omega \|_{p^\ast} + \| v'_\Omega \|_{p^\ast} - \| v_\Omega + v'_\Omega\|_{p^\ast}\right) \\
        & = 2\cstLip\frac{\delta r}{s}\left( s^{1/p^\ast} - (s- \|v_\Omega - v'_\Omega\|_0)^{1/p^\ast} \right),
    \end{align*}
    where $\|v_\Omega - v_\Omega'\|_0$ denotes the $\ell_0$ norm counting the number of different entries between $v_\Omega$ and $v_\Omega'$ i.e., the Hamming distance between $v_\Omega$ and $v_\Omega'$. 
    
    Now, it is sufficient to provide packing of $\left\{v\in\{ \pm 1\}^d : \supp(v)\subseteq \Omega\right\}$ that restricts $\|v_\Omega - v_\Omega'\|_0$ since, if $v,v'$ are not supported on $\Omega$ we have that $\mathrm{d}_{\mathrm{opt}}(v,v', \mathcal{B}_p) = 0$.
    
    It is sufficient to restrict $\| v_\Omega- v_\Omega'\|_1 \geq \frac{s}{2}$ as this implies also that $\| v_\Omega- v_\Omega'\|_0\geq \frac{s}{2}$. We can use the Gilbert-Varshimov bound \citep[Lemma 7.5]{Duchi2023Lecture}, that gives an $\frac{s}{2}$ $\ell_1$-packing of $\left\{v\in\{ \pm 1\}^d : \supp(v)\subseteq \Omega\right\}$ of size at least $\mathrm{exp(s/8)}$. Let $\mathcal{V}$ be the packing, we have that
    \begin{equation*}
        \forall v,v'\in\mathcal{V}\quad \mathrm{s.t.}\quad v\neq v':  \mathrm{d}_\mathrm{opt}(v,v',\ballpzero[r]) \geq \frac{r}{2}\cstLip \delta s^{-1/p},
    \end{equation*}
    where we used that $s^{1/p^\ast - 1} = s^{-1/p}$.

    By \citep[Lemma 2, Sec.\ A.1]{levy2019necessary}, we have
    \begin{align}
        \inf_{\hat{x}\in\ballpzero[r]} \sup_{P\in\mathcal{P}}\mathbb{E}_{S\sim P^n}  \left[ \delta f(\hat x)\right] &\geq \frac{r}{4} \cstLip \delta s^{-1/p}\left(1-\frac{3n\delta^2 + \log 2}{s/8}\right). \label{eq:lowerbound_delta}
    \end{align}
    Set $\delta =  \sqrt{\frac{s}{48n}}$. Since $s\leq 48 n$ we have
    \begin{equation*}
         \inf_{\hat{x}\in\ballpzero[r]} \sup_{P\in\mathcal{P}}\mathbb{E}_{S\sim P^n}  \left[ \delta f(\hat x)\right] \geq 
                    \frac{r}{16}\cstLip \frac{s^{1/2-1/p}}{\sqrt{n}},
    \end{equation*}
    when $s \geq 32 \log(2)$. In the low-dimensional case, when $d \leq 48n$, we can set $s = d$ to recover the lower bound in \citep{levy2019necessary}. 
    
    In the high-dimensional case, when $d > 48n$ and $1+\log^{-1}(d)\leq p \leq 2$ we choose $s=1$ to get a lower bound as
    \begin{equation*}
            \inf_{\hat{x}\in\ballpzero[r]} \sup_{P\in\mathcal{P}}\mathbb{E}_{S\sim P^n}  \left[ \delta f(\hat x)\right] \geq 
                    \frac{r}{16} \frac{\cstLip}{\sqrt{n}}.
    \end{equation*}
    Otherwise, when $d > 48 n$ and $p\geq 2$, we can set $s = n$, for which $\delta = \sqrt{\frac{1}{48}} \leq 1$, to get that 
    \begin{equation*}
            \inf_{\hat{x}\in\ballpzero[r]} \sup_{P\in\mathcal{P}}\mathbb{E}_{S\sim P^n}  \left[ \delta f(\hat x)\right] \geq 
                    \frac{r}{16}\cstLip n^{-1/p}.
    \end{equation*}
\end{proof}

\section{\texorpdfstring{{Proofs for results in Section \ref{sec:blackbox_stability}}}{Proofs for results in Section \ref{sec:blackbox_stability}}}

\subsection{Proof of \cref{thm:blackbox_uniform_stabp}} \label{subsec:proof_blackbox_uniform_stabp}
\begin{proof}\linkofproof{thm:blackbox_uniform_stabp}
        Let $f_S^{(\cstUnifC)}(x) = f_S(x) + \cstUnifC\frac{\cstRegC}{p}\| x - x_0\|_p^p$ be the regularized \ERM{}. From the theorem's assumption we have that $\norm{x^\ast-x_0}_p\leq R$ for $x^\ast = \argmin_{x\in \X} f_S(x)$, which implies by \cref{lemma:distance_bound} also that $\norm{x_\cstUnifC^\ast-x_0}_p\leq R$ for $x_\cstUnifC^\ast = \argmin_{x\in \R^d} f_S^{(\cstUnifC)}(x)$. Based on the value of $p$, we consider two cases:

        \paragraph{Case 1 ($p\geq 2$):} When $p\geq 2$, by the definition of $\cstRegC = 2^{p-2}$ the regularization term $\cstUnifC\frac{\cstRegC}{p}\norm{x - x_0}_p^p$ is $(\cstUnifC, p)$-uniformly convex and, by \cref{lemma:p_norm_ball_smooth_str_convex}, it is also $(p-1)\cstUnifC \cstRegC R^{p-2}$-smooth in $\ballp[x_0][R]$ w.r.t\ $\ell_p$-norm. As a result $f^{(\cstUnifC)}_S(x) = f_S(x) + \cstUnifC\frac{\cstRegC}{p}\norm{x-x_0}_p^p$ is $(\cstUnifC,p)$-uniformly convex and $\hat{\cstSmooth}$-smooth where $ \hat{\cstSmooth} = \cstSmooth + (p-1)\cstUnifC\cstRegC R^{p-2}$. Assume $\cstUnifC \leq \cstSmooth  R^{2-p}/((p-1)\cstRegC)$ that implies  $\hat{\cstSmooth} \leq 2\cstSmooth$ and which we later show is satisfied for our choice of $\cstUnifC$ when $T$ is large enough.

        We start with $x_0$ for which we know that $\norm{ x_0 - x^\ast}_p \leq R_0$, and by \cref{lemma:distance_bound}, also that $\norm{ x_0 - x_\cstUnifC^\ast}_p \leq R_0$, where  $R_0 \defi R$. From $f_S^{(\cstUnifC)}$ being $2\cstSmooth$-smooth we have $f_S^{(\cstUnifC)}(x_0) - f_S^{(\cstUnifC)}(x^\ast_{\cstUnifC})\leq \hat{\varepsilon}_{0} \defi \cstSmooth R_0^2 $.

        The algorithm will take $i = 1, \ldots, r$ stages to achieve the final accuracy $\hat\varepsilon$. At the stage $i$, the algorithm starts with an estimate $x_{i-1}$ within the distance $\norm{x_{i-1} - x^\ast_{\cstUnifC}}_p \leq R_{i-1}$ and will output a point $x_i$ that achieves accuracy $\hat\varepsilon_i$, such that $\hat\varepsilon_{i} = \hat\varepsilon_{i-1} / 2$. From the $(\cstUnifC, p)$-uniform convexity of $f_S^{(\cstUnifC)}$ combined with the guaranteed accuracy of $x_i$ on $f_S^{(\cstUnifC)}$, we then have that the output $x_i$ of stage $i$ satisfies
        \begin{equation*}
            \hat\varepsilon_i \geq f_S^{(\cstUnifC)}(x_i) - f_S^{(\cstUnifC)}(x^\ast_{\cstUnifC}) \geq \frac{\cstUnifC}{p} \| x_i - x^\ast_{\cstUnifC}\|_p^p.
        \end{equation*}
        Thus, from the uniform convexity of $f_S^{(\cstUnifC)}(x)$, at the next stage, $i+1$, the algorithm is initialized with a point $x_i$ for which $\norm{x_{i} - x^\ast_\cstUnifC}_p \leq R_{i} \defi ( p \hat\varepsilon_i /\cstUnifC  )^\frac1p$. We can initialize the algorithm at the next stage with this $R_{i}$.

        To achieve $\hat\varepsilon$ accuracy, it is sufficient to run the algorithm for $r = \log_2( \hat{\cstSmooth} R_0^2 /(2\hat\varepsilon) ) = \log_2( \cstSmooth R_0^2 /\hat\varepsilon )$ stages, which amounts to the following number of gradient oracle calls
        \begin{align*}
            T = \sum_{i=1}^{r} \hat{T}_i &= C \sum_{i=1}^{r} \left( \frac{\hat{\cstSmooth} R_{i-1}^2}{\hat\varepsilon_{i}}  \right)^\frac{1}{\gamma}\\
            &= C \sum_{i=0}^{r-1} \left( \hat{\cstSmooth} \left( \frac{p}{\cstUnifC}\right)^\frac{2}{p} \left(\frac{1}{\hat\varepsilon_i}\right)^{1-\frac{2}{p}} \right)^\frac{1}{\gamma} \\
            &\leq C  \hat{\cstSmooth}^\frac1\gamma \left( \frac{p}{\cstUnifC} \right)^\frac{2}{p\gamma}  \sum_{i = 0}^{r-1} \left( \frac{1}{2^{r-i-1}\hat\varepsilon} \right)^{\left(1-\frac{2}{p}\right) \frac1\gamma}\\
            &= C  \hat{\cstSmooth}^\frac1\gamma \left( \frac{p}{\cstUnifC} \right)^\frac{2}{p\gamma} \left( \frac1{\hat\varepsilon} \right)^{\left(1-\frac{2}{p}\right)/\gamma} \sum_{i = 0}^{r-1} \left( \frac{1}{2^{r-i-1}} \right)^{\left(1-\frac{2}{p}\right) \frac1\gamma},
        \end{align*}
        where for the inequality we used that $\hat\varepsilon_i = \hat\varepsilon_{r-1}2^{r-1-i} \geq \hat\varepsilon 2^{r-i-1}$ for all $i \in \{1, \ldots, r\}$.

    It remains to evaluate the sum
    \begin{align*}
        \sum_{i = 0}^{r-1} \left(\left(\frac{1}{2}\right)^{\left(1-\frac{2}{p} \right)/\gamma}\right)^{r-i-1} \leq \sum_{i = 0}^{r-1} \left(1^{1/\gamma}\right)^{r-i-1} = r = \log_2( \cstSmooth R^2 /\hat\varepsilon) 
    \end{align*}
    since for a fixed $\gamma>0$ we have $(\frac{1}{2})^{(1-2/p)/\gamma} \leq 1$ when $p\geq 2$.
    
    As a result, to reach accuracy $\hat\varepsilon$, it is sufficient to make the following number of gradient oracle calls 
    \begin{equation}
        T =  C \log_2\left( \frac{\cstSmooth R^2}{\hat\varepsilon} \right)  p^{\frac{2}{p\gamma}} \cstSmooth^\frac1\gamma \cstUnifC^{-\frac{2}{p\gamma}} \hat\varepsilon^{-\left(1-\frac{2}{p}\right)/\gamma},\label{eq:thm_black_box_stability_p_iteration_bound1}
    \end{equation}
    where we used that $\hat\cstSmooth \leq 2 \cstSmooth$ when $\cstUnifC \leq \cstSmooth R^{2-p}/((p-1)\cstRegC)$.

    From \cref{lemma:stability_unif_reg} with $\cstLip = 2\cstSmooth R$ we have that, in order for the output to be uniformly stable, we need to achieve accuracy $\hat\varepsilon$ satisfying both $\hat\varepsilon \leq \cstUnifC D^p = \cstUnifC\frac{\cstRegC}{p}R^p$ and $\hat\varepsilon \leq (p/\cstUnifC)^{1/(p-1)}\left( \frac{4\cstSmooth R}{n}\right)^{p/(p-1)}$.

    When $\cstUnifC \leq \cstRegC^{\frac{1-p}{p}} R^{2-p} \frac{4p\cstSmooth}{n}$, we have that $\cstUnifC\frac{\cstRegC}{p}R^p\leq (p/\cstUnifC)^{1/(p-1)}\left( \frac{4\cstSmooth R}{n}\right)^{p/(p-1)}$ and thus the accuracy we require is $\hat\varepsilon \leq \cstUnifC \frac\cstRegC{p} R^p $. To reach an accuracy that is upper bounded as $\hat\varepsilon \leq \cstUnifC \frac\cstRegC{p} R^p $ we get from  \eqref{eq:thm_black_box_stability_p_iteration_bound1} that the number of iterations must be at least
    \begin{align}
        \begin{aligned}
        T &\geq C \log_2\left(\frac{\cstSmooth R^{2-p} p}{ \cstUnifC \cstRegC}\right)   p^{1/\gamma} \cstSmooth^\frac1\gamma  \cstUnifC^{-\frac{1}{\gamma}} \cstRegC^{\frac{2-p}{p\gamma}}   R^{\frac{2-p}{\gamma}}.\label{eq:thm_black_box_stability_p_iteration_bound2}
        \end{aligned}
    \end{align}
    To ensure \eqref{eq:thm_black_box_stability_p_iteration_bound2} holds for a given $T$, we choose
    \begin{equation*}
        \cstUnifC = \left( \frac{C}{T}\right)^\gamma  p  \cstSmooth  \cstRegC^{\frac{2-p}{p}} R^{2-p} \log_2\left(\left(\frac{T}{C}\right)^{\gamma} \cstRegC^{-2/p}\right)^{\gamma},
    \end{equation*}
    which satisfies \eqref{eq:thm_black_box_stability_p_iteration_bound2} when $T\geq C(2 \cstRegC^{2/p})^{1/\gamma}$. For this choice of $\cstUnifC$, the requirement that $\cstUnifC \leq \cstRegC^{\frac{1-p}{p}} R^{2-p} \frac{4p\cstSmooth}{n}$ is met when $T = \tilde\Omega_p\left(n^{1/\gamma} \right)$. To ensure that $\hat\cstSmooth \leq 2\cstSmooth$, we need  $\cstUnifC \leq \cstSmooth R^{2-p}/((p-1)\cstRegC)$, which for our choice of $\cstUnifC$ is satisfied when $T\geq   C \left(p(p-1) \cstRegC^\frac2p\right)^{1/\gamma}$. Joining the two conditions gives
    \begin{equation*}
        T = \tilde\Omega_p\left( n^{1/\gamma} \right).
    \end{equation*}

    For this $\cstUnifC$ the accuracy $\hat\varepsilon$  we wish to reach from \cref{lemma:stability_unif_reg} is
    \begin{equation*}
        \hat\varepsilon \leq  \cstUnifC \frac{\cstRegC}{p} R^p =  \cstRegC^{2/p}  \cstSmooth   R^2 \left( \frac{C}{T}\right)^\gamma \log_2\left(\left(\frac{T}{C}\right)^{\gamma} \cstRegC^{-2/p}\right)^{\gamma},
    \end{equation*}
    which gives us the requirement for the minimum number of stages $r = \log_2(\cstSmooth R^2 / \hat\varepsilon)$.
    
    By \cref{lemma:stability_unif_reg} with $\cstLip = 2\cstSmooth R$, for this range of $T$, the optimization error will be
    \begin{equation*}
        \mathcal{E}_\mathrm{opt} = f_S(x_T) - f_S(x^\ast) \leq \cstUnifC \frac{\cstRegC}{p} R^p =  \left( 4 \frac{C}{T}\right)^\gamma  \log_2\left(\left(\frac{T}{C}\right)^{\gamma} \cstRegC^{-2/p}\right)^{\gamma} \cstRegC^{2/p}  \cstSmooth   R^2.
    \end{equation*}
    and the stability is
    \begin{align*}
        \mathcal{E}_\mathrm{stab}(\mathcal{A}) \leq 3\left( \frac{2p}{n\cstUnifC} \cstLip^p\right)^\frac{1}{p-1} &= 3  \left( 2 \cstSmooth^{p-1} p^{p-1} \cstRegC^{\frac{p-2}{p}} R^{2(p-1)} \left(\frac{T}{4 C} \right)^{\gamma} \log_2\left(\left(\frac{T}{C}\right)^{\gamma} \cstRegC^{-2/p}\right)^{-\gamma} \frac{1}{n}  \right)^\frac{1}{p-1} \\
        &=   \tilde\bigO_p\left( \left(\frac{T^{\gamma}}{n}\right)^\frac{1}{p-1} \cstSmooth R^2 \right),
    \end{align*}
    when $T = \tilde\Omega_p\left( n^{1/\gamma}\right)$.
    
    \paragraph{Case 2 ($p\in (1,2)$):} When $p\in (1,2)$, by the choice of $\cstRegC = 2^{p-3}(1-1/p)^{p-1}$ and \cref{lemma:unif_cvxty_of_norm_p_to_the_p}, we have that $\cstUnifC\frac{\cstRegC}{p}\norm{x}_p^p$ is $(\cstUnifC , p)$-Hölder smooth globally w.r.t\ $\ell_p$-norm. From \cref{lemma:p_norm_ball_smooth_str_convex} we have that $\cstUnifC\frac{\cstRegC}{p} \norm{x}_p^p$ is $\cstUnifC \cstRegC R^{p-2} (p-1)$-strongly convex inside of  $x\in\ballpzero[R]$ w.r.t\ $\ell_p$-norm. 
    
    In the following, denote $\psi(x) = \cstUnifC\frac{\cstRegC}{p}\norm{x-x_0}_p^p$ for ease of notation. We derive the combined Hölder smoothness of the regularized $\ERM{}$ as follows
    \begin{align*}
        f_S(x) + \psi(x) - \left( f_S(y) + \psi(y)+ \innp{\nabla f_S(y) + \nabla \psi(x), x-y}\right) &\leq \frac{\cstSmooth}{2}\norm{x-y}^2_p + \frac{\cstUnifC}{p}\norm{x-y}_p^p \\
        &\leq \frac{1}{p}\left(p 2^{1-p}\cstSmooth R^{2-p} + \cstUnifC\right)\norm{x-y}_p^p,
    \end{align*}
    for all $x,y\in\ballpzero[R]$. Consequently, we have that $f_S^{(\cstUnifC)}(x)$ is $(\hat{\cstSmooth}, p)$-Hölder smooth globally, where $\hat{\cstSmooth} = p2^{1-p}\cstSmooth R^{2-p} + \cstUnifC$ and $\cstUnifC \cstRegC R^{p-2}(p-1)$-strongly convex for $x\in\ballpzero[R]$ w.r.t\ $\ell_p$-norm. Assume $\cstUnifC \leq p2^{1-p}\cstSmooth R^{2-p}$ that implies  $\hat\cstSmooth \leq p2^{2-p}\cstSmooth R^{2-p}$ and which we later show is satisfied for our choice of $\cstUnifC$ when $T$ is large enough.

    We start with $x_0$ for which we know that $\norm{ x_0 - x^\ast_\cstUnifC}_p \leq R_0 \defi R$. By Hölder smoothness of $f_S^{(\cstUnifC)}$ we have that $f_S^{(\cstUnifC)}(x_0) - f_S^{(\cstUnifC)}(x^\ast_\cstUnifC) \leq \hat{\varepsilon}_0 \defi \hat{\cstSmooth} R_0^p / p$. 
    
    The algorithm will take $i = 1, \ldots, r$ stages to achieve final accuracy $\hat{\varepsilon}$. At the stage $i$, the algorithm starts with an estimate $x_{i-1}$ within the distance $\norm{x_{i-1} - x^\ast_\cstUnifC}_p \leq R_{i-1}$ and will output a point $x_i$ that achieves accuracy $\hat{\varepsilon}_i$, such that $\hat{\varepsilon}_{i} = \hat{\varepsilon}_{i-1} / 2$. From $\cstUnifC \cstRegC R^{p-2}(p-1)$-strong convexity of $f_S^{(\cstUnifC)}$ we have that the output $x_i$ of stage $i$ satisfies
    \begin{equation*}
        \hat{\varepsilon}_i \geq f_S^{(\cstUnifC)}(x_i) - f_S^{(\cstUnifC)}(x^\ast_\cstUnifC) \geq \cstUnifC \cstRegC R^{p-2}(p-1) \| x_i - x^\ast_\cstUnifC\|_p^2.
    \end{equation*}
    Thus, from the strong convexity of $f_S^{(\cstUnifC)}$, at the next stage, $i+1$, the algorithm is initialized with a point $x_i$ for which $\norm{x_{i} - x^\ast_\cstUnifC}_p \leq R_{i} \leq (R^{2-p}\hat{\varepsilon}_i  \cstUnifC^{-1}\cstRegC^{-1} (p-1)^{-1} )^{1/2}$.

    To achieve $\hat\varepsilon$ accuracy, it is sufficient to run the algorithm for $r = \log_2( \hat{\cstSmooth} R_0^p / (p\hat{\varepsilon}) )$ stages, which amounts to the following number of gradient oracle calls
    \begin{align*}
        T = \sum_{i=1}^{r} \hat{T}_i &= C \sum_{i=0}^{r-1} \left( \frac{\hat{\cstSmooth} R_i^p}{\hat\varepsilon_i}  \right)^\frac{1}{\gamma} \\
        &= C \sum_{i=0}^{r-1} \left( \hat{\cstSmooth} \left( \frac{R^{2-p}}{\cstUnifC \cstRegC (p-1)}\right)^\frac{p}{2} \left(\frac{1}{\hat{\varepsilon}_i}\right)^{1-\frac{p}{2}} \right)^\frac{1}{\gamma} \\
        &\leq C  \hat{\cstSmooth}^\frac1\gamma \left( \frac{ R^{2-p}}{\cstUnifC\cstRegC(p-1)}\right)^\frac{p}{2\gamma}  \sum_{i = 0}^{r-1} \left( \frac{1}{2^{r-i-1}\hat{\varepsilon}} \right)^{\left(1-\frac{p}{2}\right) \frac1\gamma}\\
        &= C  \hat{\cstSmooth}^\frac1\gamma \left( \frac{ R^{2-p}}{\cstRegC\cstUnifC(p-1)}\right)^\frac{p}{2\gamma} \left(\frac{1}{\hat\varepsilon}\right)^{\left(1-\frac{p}{2}\right)/\gamma} \sum_{i = 0}^{r-1} \left( \frac{1}{2^{r-i-1}} \right)^{\left(1-\frac{p}{2}\right) \frac1\gamma}
    \end{align*}
    where for the inequality we used that $\hat{\varepsilon}_i = \hat{\varepsilon}_{r-1}2^{r-1-i} \geq \hat{\varepsilon} 2^{r-i-1}$ for all $i\in\{ 1, \ldots, r\}$.

    It remains to evaluate the sum
    \begin{align*}
        \sum_{i = 0}^{r-1} \left(\left(\frac{1}{2}\right)^{\left(1-\frac{p}{2} \right)/\gamma}\right)^{r-i-1} \leq \sum_{i = 0}^{r-1} \left(1^{1/\gamma}\right)^{r-i-1} = r = \log_2( \cstSmooth R^p /(p\hat\varepsilon)) 
    \end{align*}
    since for a fixed $\gamma>0$ we have $(\frac{1}{2})^{(1-p/2)/\gamma} \leq 1$ when $p\in(1,2)$.

    As a result, to reach accuracy $\hat{\varepsilon}$, it is sufficient to make the following number of gradient oracle calls 
    \begin{equation}
        T = \log_2\left( \frac{\cstSmooth R^p}{p\hat\varepsilon}\right) C \left(\frac{2 p}{(p-1)^{p/2}} \right)^{1/\gamma}   \cstSmooth^{1/\gamma} R^{\frac{(2-p)(2+p)}{2\gamma}} \cstRegC^{-\frac{p}{2\gamma}} \cstUnifC^{-\frac{p}{2\gamma}} \hat\varepsilon^{-\left(1-\frac{p}{2}\right)/\gamma},\label{eq:thm_black_box_stability_p12_iteration_bound1}
    \end{equation}
    where we used that $\hat\cstSmooth \leq p2^{2-p} R^{2-p} \cstSmooth$ and that $2^{1-p}\leq 2$ for $p\in (1,2)$

    From \cref{lemma:stability_unif_reg} with $\cstLip = 2\cstSmooth R$, we have that for a $\cstUnifC \cstRegC R^{p-2}(p-1)$-strongly convex regularizer to result in an output that is uniformly stable, we need to achieve accuracy $\hat\varepsilon$ satisfying both $\hat\varepsilon \leq \cstUnifC \cstRegC R^{p-2}(p-1) D^2 = \cstUnifC \cstRegC^{1+2/p} p^{-2/p} (p-1) R^p$ and $\hat\varepsilon \leq 32 \cstSmooth^2 R^{4-p}  \cstUnifC^{-1} \cstRegC^{-1} (p-1)^{-1} n^{-2}$.

    When $\cstUnifC \leq 2^{2+\frac12} \cstSmooth R^{2-p} (p^{1/p}/(p-1)) \cstRegC^{-\frac{p+1}{p}} / n$, we have that $\cstUnifC \cstRegC^{1+2/p} p^{-2/p} (p-1) R^p \leq  32 \cstSmooth^2 R^{4-p}  \cstUnifC^{-1} \cstRegC^{-1} (p-1)^{-1} n^{-2}$ and thus the accuracy we need is $\hat\varepsilon \leq\cstUnifC \cstRegC^{1+2/p} p^{-2/p} (p-1) R^p$. To reach an accuracy that is upper bounded as $\hat\varepsilon \leq \cstUnifC \cstRegC^{1+2/p} p^{-2/p} (p-1) R^p$ we get from  \eqref{eq:thm_black_box_stability_p12_iteration_bound1} that it is sufficient to perform a number of iterations at least
    \begin{align}
        \begin{aligned}
        T &\geq C  2^{1/\gamma} \left(\frac{p^{2/p}}{p-1}\right)^\frac{1}{\gamma} R^{\frac{2-p}{\gamma}} \cstSmooth^\frac1\gamma  \cstUnifC^{-\frac{1}{\gamma}} \cstRegC^{-\frac{2}{p\gamma}} \log_2\left( \frac{\cstSmooth p^{2/p-1}}{(p-1) \cstUnifC \cstRegC^{(1+2/p)}}\right).  \label{eq:thm_black_box_stability_p12_iteration_bound2}
        \end{aligned}
    \end{align}

    To ensure \eqref{eq:thm_black_box_stability_p12_iteration_bound2} holds for a given $T$, we choose
    \begin{equation}
        \cstUnifC = \left( \frac{C}{T} \right)^{\gamma} \frac{4}{p-1} \cstSmooth R^{2-p} \cstRegC^{-\frac2p} \log_2\left(\frac{1}{2p\cstRegC}\left(\frac{T}{C}\right)^\gamma\right)^{\gamma}, \label{eq:thm_black_box_stability_p12_alpha}
    \end{equation}
    where we used that $p^{p/2} \leq 2$ for $p\in(1,2)$ and require  $T\geq C (4p\cstRegC)^{1/\gamma}$. 
    
    The choice of $\cstUnifC$ in \eqref{eq:thm_black_box_stability_p12_alpha} satisfies the requirement $\cstUnifC \leq 2^{2+\frac12} \cstSmooth R^{2-p} (p^{1/p}/(p-1)) \cstRegC^{-\frac{p+1}{p}} / n$ when $T = \tilde \Omega_p (n^{1/\gamma})$. In order to satisfy $\cstUnifC \leq p2^{1-p}\cstSmooth R^{2-p}$ which we need to have  $\hat\cstSmooth \leq p2^{2-p}\cstSmooth R^{2-p}$, for the choice of $\cstUnifC$ in \eqref{eq:thm_black_box_stability_p12_alpha}, we need $T\geq 2^{3(1+1/\gamma)} C^{1/\gamma} \cstRegC^{-\frac{2}{p\gamma}} /(p-1)^{1/\gamma}$.

    With the choice of $\cstUnifC$ as in \eqref{eq:thm_black_box_stability_p12_alpha}, we aim to achieve accuracy
    \begin{equation*}
        \hat\varepsilon \leq \cstUnifC \cstRegC^{1+2/p} p^{-2/p} (p-1) R^p= 2 \left(\frac{C}{T} \right)^{\gamma} \log_2\left(\frac{1}{2p\cstRegC}\left(\frac{T}{C}\right)^\gamma\right)^{\gamma} \cstRegC \cstSmooth R^2,
    \end{equation*}
    which gives us the requirement for the minimum number of stages $r=\log_2( \cstSmooth R^p /(p\hat\varepsilon)) = \log_2(R^{p-2}/(2p\cstRegC))$. From \cref{lemma:stability_unif_reg}, we have that the optimization error will be
    \begin{equation*}
        \mathcal{E}_{\mathrm{opt}} = f_S(x_T) - f_S(x^\ast) \leq 4 \left(\frac{C}{T} \right)^{\gamma} \cstRegC \cstSmooth R^2 \log_2\left(\frac{1}{2p\cstRegC}\left(\frac{T}{C}\right)^\gamma\right)^{\gamma}
    \end{equation*}
    and the stability is
    \begin{align*}
        \mathcal{E}_\mathrm{stab}(\mathcal{A}) 
        & =\tilde\bigO_p\left( \frac{T^{\gamma} \cstSmooth R^2}{n} \right),
    \end{align*}
    when $T =\tilde\Omega_p(n^{1/\gamma})$.
\end{proof}

\subsection{Proof of \cref{corr:usolp_excess_risk_low_dim}} \label{subsec:proof_usolp_excess_risk_low_dim}
\begin{proof}\linkofproof{corr:usolp_excess_risk_low_dim}
    Let $p>1$ and denote $\hat{p} = \max\{ p, 2\}$. Let $\psi(x) = \frac{1}{2}\norm{x-x_0}_2^2$. When $x^\ast\in\ballp[x_0]$, we have due to the choice of $\psi(x)$ that
    \begin{equation*}
        \norm{x^\ast-x_0}_p \geq \sqrt{2} d^{1/\hat{p}-1/2} \left(\frac12\norm{x^\ast-x_0}_2^2 \right)^{1/2},
    \end{equation*}
    which implies that we can apply \cref{thm:risk_upper} with $D = \sqrt{1/2} d^{1/2-1/\hat{p}}R$.

    The regularized objective $f^{(\cstUnifC)}_S$ is $(\cstSmooth + \cstUnifC)$-smooth. Assume that $\cstUnifC \leq \cstSmooth$, which we soon show that it holds for $n$ large enough, such that the objective $f^{(\cstUnifC)}_S$ is $2\cstSmooth$-smooth. 
    
    We apply \cref{thm:risk_upper} using the fact that $\cstLip = 2\cstSmooth R$ for $x\in\ballp[x_0]$ as explained in \eqref{eq:smooth_to_lipschitz}, which requires to choose $\cstUnifC = 2^{1-1/2}\sqrt{3} D^{-1} \cstLip =  2 \sqrt{3} d^{1/\hat{p}-1/2} n^{-1/2} \cstSmooth$ where $\cstLip$-Lipschitz and $\cstSmooth$-smooth constants of the loss are w.r.t\, $\ell_p$-norm. To ensure that $\cstUnifC \leq \cstSmooth$ we assumed earlier, we need $n\geq 12 d^{2/\hat{p}-1}$.

    By the developments in the convergence analysis of USOLP in \cref{subsec:proof_blackbox_uniform_stabp} for $p = 2$, the equation \eqref{eq:thm_black_box_stability_p_iteration_bound1} gives that the total number of gradient oracle calls for this choice of $\cstUnifC$ is 
    \begin{equation*}
        T = \tilde\Omega_p\left( \left(\frac{n^{1/2}}{d^{1/\hat{p}-1/2}}\right)^{1/\gamma}\right).
    \end{equation*}
    and the bound on the excess risk is
    \begin{equation*}
        \mathbb{E}_S \left[ \delta f(\hat{x}_\cstUnifC)\right] =\bigO_p\left(  \cstSmooth R^2 \frac{d^{1/\hat{p}-1/2}}{n^{1/2}}\right),
    \end{equation*}
    where we used that $D = \sqrt{1/2} d^{1/2-1/\hat{p}}R$.
\end{proof}

\subsection{Proof of \cref{corr:usolp_excess_risk_high_dim}} \label{subsec:proof_usolp_excess_risk_high_dim}
\begin{proof}\linkofproof{corr:usolp_excess_risk_high_dim}
    \textbf{Case 1 ($p \geq2$):} We choose $\cstUnifC = (\frac{3}{p-1})^{1-\frac1p} 2^{2/p} p^{1+\frac1{p(p-1)}} \cstRegC^{\frac{1-p}{p}} \left(\frac{1}{n}\right)^\frac{1}{p} R^{2-p} \cstSmooth$ as given in the proof in \eqref{eq:alpha_minimum} with $\cstLip = 2\cstSmooth D$ and $D = (\cstRegC/p)^{1/p} R$. When $n \geq 3^{p-1} 4 (p-1) p^{p+\frac{1}{p-1}}\cstRegC$, we have that $\cstUnifC \leq \cstSmooth R^{2-p}/((p-1)\cstRegC)$ and thus $\hat \cstSmooth \leq 2\cstSmooth$ as given in the proof of \cref{thm:blackbox_uniform_stabp} for $p\geq2$, where $\hat\cstSmooth$ is the smoothness of the regularized empirical risk objective.
    
    By the developments in the convergence analysis of USOLP given in \cref{subsec:proof_blackbox_uniform_stabp} for case $p\geq 2$, the equation \eqref{eq:thm_black_box_stability_p_iteration_bound1} gives that the number of iterations for $x_T$ to be $\hat\varepsilon$-accurate on $f_S^{(\cstUnifC)}$ is
    \begin{equation*}
        T = \tilde\Omega_p \left(  \cstSmooth^{\frac1\gamma(1-\frac2p)} R^{2\frac{p-2}{p\gamma}} \cstRegC^{2\frac{p-1}{p^2\gamma}} n^{\frac{2}{p^2\gamma}} \hat\varepsilon^{-\left(1-\frac{2}{p}\right)/\gamma} \right),
    \end{equation*}
    which, since we need to ensure $\hat\varepsilon \leq 6 \cstSmooth (\cstRegC/p)^{1/p}R^2 / n^{1+1/p}$ in order for \cref{thm:risk_upper} to apply, implies we need to have number of iterations as
    \begin{align*}
        T &=\tilde \Omega_p\left( \left( n^{1-1/p} \right)^{1/\gamma} \right).
    \end{align*}
    Consequently, \cref{thm:risk_upper} upper bounds the excess risk of $x_T$ as $\mathbb{E}_S[\delta f(x_T)] \leq 16  (\cstRegC/p)^{1/p} \cstSmooth R^2 / n^{1/p}$ for $p\geq 2$.

    \textbf{Case 2 ($p \in (1,2)$):} We have that $f_S^{(\cstUnifC)}$ is $(\hat\cstSmooth, p)$-Hölder smooth and $\cstUnifC \cstRegC (p-1) R^{p-2}$-strongly convex in $\mathcal{B}_{\norm{\cdot}_p}(x_0, R)$. We set $\cstUnifC = 4 \sqrt{3} \cstRegC^{-3/2}  \sqrt{\frac{1}{n}} R^{2-p}\cstSmooth/(p-1)$, which is equivalent to $\cstUnifC (p-1) D^{p-2} = 4 \sqrt{3} \sqrt{\frac{1}{n}}  \cstSmooth $, and corresponds to the optimal choice of the strong-convexity constant in \cref{thm:risk_upper} when the regularizer is $\cstUnifC \cstRegC (p-1) R^{p-2}$-strongly convex. By the developments in the convergence analysis of USOLP given in \cref{subsec:proof_blackbox_uniform_stabp} for case $p\in (1,2)$, the equation \eqref{eq:thm_black_box_stability_p12_iteration_bound1} gives that the number of iterations for $x_T$ to be $\hat\varepsilon$-accurate on $f_S^{(\cstUnifC)}$ is
    \begin{equation*}
        T = \tilde\Omega_p \left(  \cstSmooth^{\frac{2-p}{2\gamma}} R^{\frac{2-p}{\gamma}} n^{\frac{p}{4\gamma}} \hat\varepsilon^{-\left(1-\frac{p}{2}\right)/\gamma}\right)
    \end{equation*}
    when $\hat\cstSmooth \leq p2^{2-p} \cstSmooth R^{2-p}$. Since we need to ensure $\hat\varepsilon \leq 6 \cstSmooth  R^2 (\cstRegC / p)^{1/p} / n^{3/2}$ in order for \cref{thm:risk_upper} to apply, implies we need to have number of iterations as
    \begin{align*}
        T &= \tilde\Omega_p \left(  n^{\frac{3-p}{2\gamma}}\right)
    \end{align*}
    
    When $n\geq 48 p^{2p-4} (p-1)^{2p-4}$, we have that $\cstUnifC \leq \cstSmooth R^{1-p}p^{2-p} (p-1)^{p-1}$, which guarantees that $\hat\cstSmooth \leq p 2^{2-p} R^{2-p}\cstSmooth$ by the developments in \cref{subsec:proof_blackbox_uniform_stabp} for $p\in(1,2)$.
    
    Consequently, \cref{thm:risk_upper} upper bounds the excess risk of $x_T$ as $\mathbb{E}_S[\delta f(x_T)] \leq 16 (\cstRegC/p)^{1/p} \cstSmooth R^2 / n^{1/2}$ for $p\in (1,2)$.
\end{proof}

\subsection{Proof of \cref{thm:diakonikolas_stability_bound} \label{subsec:proof_diakonikolas_stability_bound}}
\begin{proof}\linkofproof{thm:diakonikolas_stability_bound}
Consider the Generalized AGD+ algorithm from \citep{diakonikolas2021complementary} for minimizing $f_S^{(\cstUnifC)} = f_S(x) + \cstUnifC \psi(x)$ , where $f_S(x):\mathcal{X} \rightarrow \R$ is $\cstSmooth$-smooth, $\cstUnifC \psi(x)$ is $(\cstUnifC,p)$-uniformly convex, and we assume that $\cstSmooth^{p/2}\geq \cstUnifC$. From \citep[(eq. 15)]{diakonikolas2021complementary}, we know the algorithm takes
\begin{align}
    \begin{aligned}
    T &\geq c\left(\min \left\{ \left(\frac{1}{\hat\epsilon}\right)^{\frac{p-2}{p+2}} \left(\frac{\cstSmooth^{p/2}}{\cstUnifC}\right)^{\frac{2}{p+2}} \log\left( \frac{\cstSmooth D}{\hat\epsilon}\right), \left(\frac{\cstSmooth}{\hat\epsilon}\right)^{\frac{p}{p+2}} \left(\frac{D^{p/2}}{\cstUnifC}\right)^{\frac{2}{p+2}}\right\}\right) \\
       &\Big(\leq c \left(\frac{1}{\hat\epsilon}\right)^{\frac{p-2}{p+2}} \left(\frac{\cstSmooth^{p/2}}{\cstUnifC}\right)^{\frac{2}{p+2}} 
    \log\left(\frac{\cstSmooth D^p}{\hat\varepsilon}\right)\Big)
    \end{aligned} \label{eq:optimization_error1}
\end{align}
gradient oracle calls for some constant $c>0$ in order to achieve $f^{(\cstUnifC)}(x_T) - f^{(\cstUnifC)}(x_\cstUnifC^*)\leq \hat\epsilon$. Due to \cref{lemma:distance_bound} we have that $D^p \geq D_\Psi(x_\cstUnifC^*,x_0)$ and for the second inequality we picked the first argument of the minimum. \emph{(Here, in the notation of \citep{diakonikolas2021complementary}, we take $\phi(u) = \frac1\cstUnifC D_\Psi(u, x_0)$ and upper bound it with $D/\cstUnifC$.)}

In order to apply \cref{lemma:stability_unif_reg} with $\cstLip = 2\cstSmooth R$ and $D = (\cstRegC/p)^{1/p} R$, we need the Generalized AGD+ algorithm to reach $\hat{\varepsilon}$-minimizer, satisfying both: $\hat\varepsilon \leq \cstUnifC (\cstRegC/p) R^p$ and $\hat\varepsilon \leq \left(p/\cstUnifC\right)^{1/(p-1)} \left(\frac{4 \cstSmooth R}{n}\right)^{p/(p-1)}$.

When $\cstUnifC \leq \cstRegC^\frac{1-p}{p}R^{2-p}  \frac{4p\cstSmooth}{n}$, we have that $\cstUnifC \frac{\cstRegC}{p} R^p \leq \left(p/\cstUnifC\right)^{1/(p-1)} \left(4\cstSmooth R/n\right)^{p/(p-1)}$. Then the accuracy we need is $\hat\varepsilon \leq \cstUnifC \frac{\cstRegC}{p} R^p$ and the number of required iterations is lower bounded as
\begin{align}
    \begin{aligned}
    T &\geq c \left( \frac{1}{\cstUnifC}\right)^{\frac{p}{p+2}}  \left(\frac{\cstRegC}{p}\right)^{\frac{p-2}{p+2}} R^{-p\frac{p-2}{p+2}} \cstSmooth^\frac{p}{p+2} \log\left( \frac{p \cstSmooth}{\cstUnifC}\right) \label{eq:iteration_bound2}
    \end{aligned}.
\end{align}

For a given $T$, we can choose 
\begin{equation}
    \cstUnifC = \left(\frac{c}{T}\right)^{1+\frac2p} \left(\frac{p}{\cstRegC}\right)^{1-\frac{2}{p}} R^{2 - p} \cstSmooth \left(\log\left( \left(\frac{T}{c}\right)^{1+2/p} p^{2/p} R^{p-2}\right)\right)^{1+\frac2p}, \label{eq:genagd_alpha_choice}
\end{equation}
which for $T \geq c e^\frac{p}{p+2} (p/\cstRegC)^{-2/(p+2)} R^{-p\frac{p-2}{p+2}}$, satisfies \eqref{eq:iteration_bound2}. 

The required bound on $\cstUnifC \leq \cstRegC^\frac{1-p}{p} R^{2-p} \frac{4p\cstSmooth }{n}$ is satisfied when
\begin{equation}
    T\geq c \left(\frac{n }{4 }\right)^\frac{p}{p+2} \left(\frac{\cstRegC}{p}\right)^{\frac{2}{p+2}} \log\left( \left(\frac{T}{c}\right)^{1+\frac2p} p^{2/p} R^{p-2}\right). \label{eq:iteration_bound3}
\end{equation}

By \citep[Lemma 4]{baudry2024multi}, we have that for $x\geq3$ and constants $A,B>0$, if $x\geq 3 \frac{A}{B} \log(A)$, then also $x \geq A \log(Bx)$. Applying the result to \eqref{eq:iteration_bound3}, yields that the Generalized AGD+ algorithm reaches the required accuracy when the number of iterations lower bounded as
\begin{align*}
    T &\geq 3 p^{-\frac{4}{p+2}} \left(1+\frac{2}{p}\right) c^2 \left( \frac{n}{4 }\right)^{\frac{p}{p+2}} R^{2-p} \log \left( c \left(\frac{n}{4}\right)^\frac{p}{p+2} p^{-\frac{2}{p+2}} \right)\\
    & = \tilde\Omega_p\left( c^2 R^{2-p} n^{\frac{p}{p+2}} \right).
\end{align*}
 
For this range of $T$ and the choice of $\cstUnifC$ in \eqref{eq:genagd_alpha_choice}, the optimization error is
\begin{align*}
    \mathcal{E}_{\mathrm{opt}} = f_S(x_T) - f_S(x^\ast) \leq 2\cstRegC\cstUnifC R^p /p &= 2^{p-1} \left(\frac{c}{T}\right)^{1+\frac2p} p^{-\frac{2}{p}} D^2 \cstSmooth \left(\log\left( \left(\frac{T}{c}\right)^{1+2/p} p^{2/p} D^{p-2}\right)\right)^{1+\frac2p} \\
    &= \mathcal{\tilde O}_p\left( D^2 \cstSmooth \left(\frac1T\right)^{1+2/p} \right)
\end{align*}
and the stability is 
\begin{align*}
    \mathcal{E}_{\mathrm{stab}}(\mathcal{A})&\leq 3 \left( \frac{2p}{n\cstUnifC} \cstLip^p \right)^{1/(p-1)} \\
    &= 3 \left( \frac{2^{1+p}p^{2/p}\cstSmooth^p D^p}{\cstSmooth n D^{2-p}} \left(\frac{T}{c}\right)^{1+2/p} \right)^\frac{1}{p-1}  \left(\log\left( \left(\frac{T}{c}\right)^{1+2/p} D^{p-2}\right)\right)^{-(1+\frac2p)\frac{1}{p-1}} \\
    & = \mathcal{\tilde O}_p\left( \left(\frac{T^{1+2/p}}{n}\right)^{1/(p-1)} \cstSmooth D^2 \right).
\end{align*}
\end{proof}

\subsection{Proof of \cref{thm:diakonikolas_generalization_bound} \label{subsec:proof_diakonikolas_generalization_bound}}
\begin{proof}\linkofproof{thm:diakonikolas_generalization_bound}
    For $p \geq 2$, choose $\cstUnifC = (\frac{3}{p-1})^{1-\frac1p} 2^{2/p} p^{1+\frac1{p(p-1)}} \cstRegC^{\frac{1-p}{p}} \left(\frac{1}{n}\right)^\frac{1}{p} R^{2-p} \cstSmooth$ as in \cref{thm:risk_upper}, and plug it into the convergence rate of the Generalized AGD+ given in \eqref{eq:optimization_error1} in the proof in \cref{subsec:proof_diakonikolas_stability_bound}. We get that the number of required iterations for $\hat\varepsilon$ accuracy is
    \begin{equation*}
        T\geq c \left( \frac{p-1}{3}\right)^{2\frac{p-1}{p+2}} 2^{-\frac{4}{p(p+2)}} p^{-\frac{2}{p+2}} n^{\frac{2}{p(p+2)}} R^{2\frac{p-2}{p+2}} \cstSmooth^{\frac{p-2}{p+2}} \left(\frac{1}{\hat\epsilon}\right)^{\frac{p-2}{p+2}}  
    \log\left(\frac{\cstSmooth R^p}{\hat\varepsilon}\right),
    \end{equation*}
    and to achieve accuracy $\hat\varepsilon \leq 12 \cstSmooth R^2 / n^{1+\frac1p}$, we need
    \begin{align*}
        T&= \tilde\Omega_p\left(  n^{\frac{p-1}{p+2}}\right)
    \end{align*}
    Consequently, \cref{thm:risk_upper} upper bounds the excess risk of $x_T$ as $\mathbb{E}_S[\delta f(x_T)] \leq 16 \cstSmooth R^2 / n^{1/p}$ for $p\geq 2$.
\end{proof}
\end{document}